\documentclass[oneside,11pt]{article}
\usepackage{caption}
\usepackage{blindtext}

\usepackage[abbrvbib, preprint]{jcustom}

\usepackage{lastpage}

\usepackage[utf8]{inputenc} 
\usepackage[T1]{fontenc}           
\usepackage{booktabs} 
\usepackage{hyperref}

\hypersetup{
    colorlinks=true,
    allcolors=blue
}

\usepackage{amsfonts,amsmath,amsthm,amssymb}       
\usepackage{nicefrac}       
\usepackage{microtype}      
\usepackage{xcolor}         
\usepackage{pifont}
\usepackage{tikz}
\usepackage{subcaption}

\usetikzlibrary{arrows.meta}
\usepackage{wrapfig}
\usepackage[presets={vec-cev}]{letterswitharrows}
\usepackage{cleveref}
\AddToHook{cmd/appendix/before}{
    \crefalias{section}{appendix}%
    \crefalias{subsection}{appendix}
}
\usepackage{graphicx}
\usepackage{amsfonts}
\usepackage{mathtools}
\usepackage{booktabs}
\usepackage{nicefrac}

\newtheorem{theorem}{Theorem}
\theoremstyle{remark}

\newcommand{\R}{\mathbb{R}}

\definecolor{src}{RGB}{76,114,176}
\definecolor{tgt}{RGB}{221,132,82}

\tikzset{
  traj/.style={black!75, line width=0.9pt, -{Latex[length=1.8mm]}},
  noisytraj/.style={black!70, line width=0.8pt, -{Latex[length=1.6mm]}, opacity=0.75},
  margsrc/.style={fill=src!12, draw=src!55, line width=0.5pt},
  margtgt/.style={fill=tgt!12, draw=tgt!55, line width=0.5pt},
  sp/.style={fill=src, draw=src},
  tp/.style={fill=tgt, draw=tgt},
}

\title{One-Step Generative Surrogate Models via Block-Triangular Joint Drifting}

\author{%
  Nicholas Geissler\\
   \addr Courant Institute of Mathematical Sciences, New York University\\
   \AND
 \name Shreya Jha \\
  \addr Courant Institute of Mathematical Sciences, New York University\\
  \AND
 \name Ricardo Baptista \\
  \addr Statistical Sciences, University of Toronto\\
   \AND
  Benjamin Peherstorfer\\
   \addr Institute of Mathematics, EPFL
}

\begin{document}

\maketitle

\begin{abstract}
Drifting provides a direct route to one-step generative models, but applying it directly to stochastic transition modeling requires multiple samples of the next state conditioned on the same current state. Standard trajectory data, however, typically provide only one realized next state for each observed current state and therefore do not provide an empirical approximation of the corresponding conditional distribution over possible next states. We introduce block-triangular joint drifting, which instead applies a projected drift field to the empirically accessible joint distribution of consecutive states. Importantly, the block-triangular architecture preserves the current-state marginal while making its second component a direct sampler of the conditional distribution of possible next states. The resulting surrogate generates stochastic trajectories with one model evaluation per time step, without auxiliary generative steps between time steps. Numerical experiments demonstrate accurate marginal and trajectory-dependent statistics and favorable accuracy-cost tradeoffs compared with deterministic, diffusion-, flow-, and distillation-based generative surrogate models.
\end{abstract}



\section{Introduction}

Learning low-cost reduced models of stochastic systems is a ubiquitous task for enabling uncertainty quantification, ensemble forecasting, design, and other outer-loop applications. 
While for deterministic systems it is sufficient to learn a function that maps the current state $x(t)$ to the state $x(t + 1)$ at the next time step, see, e.g., \cite{PEHERSTORFER2016196,Fresca2021,Lu2021,JMLR:v24:21-1524,Regazzoni2024,BATLLE2024112549}, for stochastic systems, one is typically interested in learning a model to rapidly sample from the transition law represented as a conditional distribution $p_t(\cdot \,|\, x(t))$ given a realization $x(t)$ of the current state $X(t)$. Once such a model is available, stochastic trajectories can be generated autoregressively. 

Diffusion- and flow-based generative modeling provides flexible mechanisms for modeling complex conditional distributions such as transition laws  \cite{sohldickstein2015deep,ho2020denoising,song2021score,pmlr-v180-shi22a,lipman2023flow,tong2024conditional,albergo2025stochastic}. These approaches have been extended to probabilistic forecasting and simulation of dynamical systems \cite{chen2024probforecast,cachay2023dyffusion,KOHL2026108641}, where a generative model conditioned on the current state can represent a distribution over possible future states. Their main drawback is computational costs of generating trajectories. Generating a new sample from the transition law typically requires multiple denoising or flow-integration steps. Thus, a rollout of $T$ time steps with $K$ generative steps per time step requires $KT$ model evaluations, which becomes costly when generating large trajectory ensembles.

This has motivated work on one- and few-step generative models, often building on diffusion- and flow-based modeling. For example, progressive distillation compresses multi-step diffusion samplers into students that use increasingly fewer steps \cite{salimans2022progressive}; consistency models enable one- or few-step sampling through either distillation or direct training \cite{song2023consistency}; and ReFlow iteratively straightens generative trajectories to reduce the number of generation steps needed per time step \cite{liu2023rectified}. Other approaches such as inductive moment matching \cite{zhou2025imm}, MeanFlow \cite{geng2025mean}, flow-map self-distillation \cite{boffi2025how}, and stochastic lifting \cite{berman2026stochastic} instead target one- or few-step generation directly, without requiring a pretrained teacher, even though MeanFlow can be applied as a distillation method as well. These methods reduce generation costs but when used as teacher-based distillation schemes for stochastic dynamics, they require first learning a multi-step conditional generative model and subsequently compressing it. Additionally, inherent to the distillation is that accuracy is traded for runtime speedups. 

A complementary line of work learns deterministic reduced dynamics for stochastic systems, enabling fast rollouts without an auxiliary generative sampling process. Action Matching \cite{neklyudov2023action,berman2024parametric} and DICE \cite{blickhan-berman-stuart-etal:2025} learn population dynamics consistent with the evolving time marginals, while, e.g.,  \cite{schwerdtner2026twoparameter,berman2026ngif} extend this viewpoint to more general, non-gradient transport fields. Because time marginals do not determine the stochastic transition law, such population models cannot in general recover trajectory-dependent statistics. Yet another complementary line of work learns reduced models of stochastic or uncertain dynamics by identifying low-dimensional stochastic differential equations with learned drift and diffusion operators \cite{freitag2025learning} or by combining dimensionality reduction with probabilistic latent dynamics to construct generative reduced-order models \cite{conti2026veni}.

In this work we instead build on drifting, a generative modeling concept for learning one-step generative models directly \cite{deng2026drifting,han2026wflow,turan2026generativedriftingsecretlyscore,lai2026unifiedviewdriftingscorebased}. Drifting starts from the distribution induced by the current model and iteratively transports its samples toward samples from a target distribution using a distribution-dependent drift field. The transport is performed during training. After convergence, the model itself maps reference noise to a target sample in one model evaluation. In particular, drifting does not require numerical integration of an auxiliary generative dynamics when generating samples. 

However, applying drifting directly to a transition law $p_t(\cdot\mid x(t))$ raises a difficulty that does not arise in the same way for conditional diffusion or flow-matching objectives.
Such methods can be trained from paired samples $(X(t),X(t+1))$, each pair providing a sample for a regression objective for the conditional score or velocity field, and repeated realizations of $X(t+1)$ at exactly the same conditioning state are not required.
Drifting is different because its update is defined through a distribution-dependent field and therefore requires an empirical sample representation of the target distribution. If one sets $p_t(\cdot\mid X^{(i)}(t))$ as target, trajectory data provide only the single successor $X^{(i)}(t+1)$, so the corresponding empirical target is the point mass $\delta_{X^{(i)}(t+1)}$. It contains no direct information about the spread or shape of the transition law.

Our key step is to change the distribution to which drifting is applied. Rather than drifting toward the transition law directly, which is a conditional distribution, we lift the problem to the joint law $p_t(X(t),X(t+1))$, for which the observed transition pairs provide samples directly. Learning joint distributions in order to obtain conditional generators has a long history in measure transport. In particular, triangular transport maps expose conditional distributions through their structure and have been developed for sample-based Bayesian inference, density estimation, and nonlinear state-space models \cite{Marzouk2017,JMLR:v19:17-747,doi:10.1137/20M1312204,Baptista2024}. Closest to our construction, block-triangular transport methods learn conditional samplers from samples of a joint distribution, thereby sharing information across conditioning values without requiring repeated samples at each value \cite{baptista2024conditional}. Joint generative models can also be conditioned at inference, including in diffusion-based scientific forecasting and simulation-based inference \cite{shysheya2024conditional,gloeckler2024allinone}. Related conditional optimal-transport and flow-matching approaches likewise exploit joint or block-triangular constructions for conditional generation \cite{kerrigan2024dynamic,zhai2026conditional}. Joint models of consecutive states have also been proposed for generative forecasting, with conditional predictions extracted from collections of generated joint samples \cite{wyrod2026generative}.

Our use of the joint law differs in its role for drifting and in the resulting inference procedure. We constrain the joint generator to the block-triangular form and project the joint drifting field onto the second block. The joint law thus supplies the empirical target needed by drifting during training, while the triangular block structure makes the learned second component a direct sampler of the transition law. After training, no conditioning procedure, reverse diffusion, flow integration, or search over joint samples is required; one neural-network evaluation produces a sample of the transition law.

We demonstrate on numerical experiments with stochastically forced Burgers' and Navier-Stokes equations that the proposed approach generates diverse and accurate trajectories with one (neural-network) model evaluation per time step. In particular, we show that our approach matches or improves upon more expensive conditional flow and autoregressive diffusion baselines, as well as upon one- and few-step MeanFlow and ReFlow-based distilled models.

\section{Preliminaries and problem formulation}
We recapitulate preliminaries and state the problem of applying drifting to transition laws.

\subsection{Setup}
Consider an $n$-dimensional stochastic process $X(t)$ that is defined over $\mathbb{R}^n$ and  discrete time $t \in \{0, \dots, T\} \subset \mathbb{N}$. 
We denote the corresponding transition law as $X(t + 1) \sim p_t(\cdot \,|\, X(t))$ with initial $X(0) \sim \mu_0$. Note that the transition law can depend on time $t$, i.e., we are not restricting the following to time-homogeneous transition laws.
The joint law of $X(t)$ and $X(t + 1)$ is denoted as  
\[
p_t(X(t), X(t + 1)) = \mu_t(X(t))p_t(X(t + 1) | X(t)),
\]
with the time marginal $\mu_t$ recursively defined as $\mu_{t+1}(y) = \int p_t(y | x) \mu_{t}(x)\mathrm dx$.
In the following, we have access to training data obtained from the process $X(t)$,
\begin{equation}\label{eq:data}
\mathcal{D} = \{X^{(i)}(t) \,|\, i = 1, \dots, N, t = 0, \dots, T\} \subset \mathbb{R}^n\,, 
\end{equation}
which consists of $i = 1, \dots, N$ trajectories $X^{(i)}(0), \dots, X^{(i)}(T)$ over $t = 0, \dots, T$ time steps.

\subsection{One-step generative surrogate modeling}
We seek a
    map  $g: \R^n \times \R^d \times \mathbb{N} \to \R^n$ such that for $\mu_t$-a.e. $x_t$, we have
    \begin{equation}\label{eq:conditionalmap}
        g(x(t),\cdot,t)_{\sharp}\pi
        =
        p_t(\cdot\,|\,x(t))\,,
    \end{equation}
    where $\pi$ is a suitable reference distribution such as a standard normal. 
The condition \eqref{eq:conditionalmap} implies that if $z\sim\pi$ is a sample from the reference $\pi$, then evaluating the map $g$ at $z$,
    \[
        g(x(t),z,t) \sim p_t(\cdot\,|\,x(t)),
    \]
    gives a sample of the transition law $p_t(\cdot \,|\, x(t))$. 
Correspondingly, we refer to $g$ as a one-step map because a single evaluation $g(x(t),z,t)$ produces a sample from the transition law at the next time point. In particular, no numerical integration of auxiliary dynamics or a sequence of intermediate generative steps are required.

\subsection{Drifting schemes for time marginals}
One approach 
    for learning one-step generators to sample from a target distribution $\eta$ is given by drifting schemes, first introduced in \cite{deng2026drifting}, and further explored in, e.g.,  \cite{han2026wflow,turan2026generativedriftingsecretlyscore,lai2026unifiedviewdriftingscorebased}.
For a parametrized function $f_{\theta}: \R^d  \to \R^n$, where $\theta$ denotes the vector of, e.g., neural-network weights, define the model distribution induced by the pushforward $q_{\theta}=(f_{\theta})_{\sharp}\pi$. 
Drifting schemes iteratively update the parameters $\theta_j$ of $f_{\theta_j}$ over iterations $j = 0, 1, 2, \dots$ such that the model distribution $q_{\theta_j} = (f_{\theta_j})_{\sharp}\pi$ improves the match to the target distribution $\eta$.
The iterative updating is achieved via a drift field $V_{\eta,q_{\theta_j}}: \mathbb{R}^n \to \mathbb{R}^n$, which determines how samples $x \sim q_{\theta_j}$ should be transported, 
    \[
    T_{h,q}(x) = 
    x + h V_{\eta, q}(x)\,,
    \]
    where $h > 0$ is a step size. 
The drifting fields must satisfy $V_{\eta, \eta} = 0$ so that the target distribution $\eta$ is a fixed  point $T_{h,\eta}(x) = x$ of $q \mapsto T_{h,q}$. 
On the parameter level, if at iteration $j$ we have $x = f_{\theta_j}(z)$ with a sample $z \sim \pi$, then the transported sample is 
    \[
    \tilde{x} = T_{h,q_{\theta_j}}(f_{\theta_j}(z)) = f_{\theta_j}(z)  + hV_{\eta, q_{\theta_j}}(f_{\theta_j}(z)).
    \]
Because at iteration $j$ we have $f_{\theta_j}(z) \sim q_{\theta_j}$, the distribution of the transported samples is $\tilde{q}_{j + 1} = (T_{h,q_{\theta_j}})_{\sharp}q_{\theta_j}$.
This defines an iteration: draw samples from the current $q_{\theta_j}$ with $f_{\theta_j}$, transport the samples with the drift field, and then fit $\theta_{j + 1}$ so that $f_{\theta_{j + 1}}$ generates samples close to the transported samples, which is achieved with the loss
    \begin{equation}\label{eq:vanilladriftloss}
        \mathcal{L}(\theta; \theta_j) = \mathbb{E}_{z \sim \pi} \left[\|f_{\theta}(z) - \left(f_{\theta_j}(z)+hV_{\eta,q_{\theta_j}}(f_{\theta_j}(z))\right)\|_2^2 \right].
    \end{equation}
In this manner, each gradient step approximates one fixed-point iteration, with $f_{\theta_j}$ parameterizing the current iterate.

\subsection{Problem formulation: Conditional drifting from trajectory data}
The drifting field $V_{\eta,q_{\theta_j}}$ depends on the target distribution
$\eta$; however, in the data-driven setting, which we consider here, 
the drifting field is constructed from a collection of samples of the target $\eta$ together with samples from the current model
distribution $q_{\theta_j}$. 
In particular, the samples from the target distribution provide the empirical approximation of
$\eta$ used by the drifting procedure.
If we now consider the direct application
of drifting to learn how to sample from the transitional law, this would mean to condition on a state $x(t)$ and take $p_t(\cdot\,|\,x(t))$ to be the target distribution $\eta$.  
The difficulty is that the trajectory data in \eqref{eq:data} do not provide multiple samples from $p_t(\cdot \,|\, x(t))$. Instead, for each observed conditioning state $ x(t) = X^{(i)}(t)$, the training data contain only its single realized successor $X^{(i)}(t+1)$.
Consequently, conditioning the empirical training data on an observed
state $X^{(i)}(t)$ does not provide the drifting procedure with samples that
characterize the corresponding  conditional
$p_t(\cdot\,|\,X^{(i)}(t))$. In particular, just having access to one realization of $X^{(i)}(t + 1)$ for a given $X^{(i)}(t)$
provides no information about the spread, shape, or possible multimodality
of the conditional distribution.
We therefore seek a drifting scheme that can learn the transition law from trajectory data without requiring multiple  samples from the conditional distribution for a given conditioning state.

\section{Block-Triangular Joint Drifting}
We propose to apply drifting to the joint transition law, for which transition-pair samples are available, while restricting the generator to a block-triangular form from which the desired conditional one-step map can be extracted.

\subsection{Targeting the joint law}
For a fixed time $t$, the training data \eqref{eq:data} contains the transition pairs  $\{(X^{(i)}(t),X^{(i)}(t+1))\}_{i=1}^{N}$, which are samples from the joint law $p_t(\cdot,\cdot)$.
Thus, the training data contains multiple samples from the joint law $p_t(\cdot,\cdot)$ for a fixed $t$ as long as $N > 1$.
Rather than separating the transition pairs at time $t$ via conditioning, we use the joint law $p_t$ as the target distribution of the drifting procedure. In this way, all available
    transition pairs at time $t$ contribute samples to the same drifting target.
We introduce the corresponding parametrized map $
        f_\theta:
        \R^n\times\R^d\times\mathbb{N}
        \to
        \R^{2n}
    $
    and model joint distribution
\begin{equation}\label{eq:TargetJointatTheta}
        q_\theta(t)
        =
        f_\theta(\cdot,\cdot,t)_{\sharp}
        (\mu_t\times\pi).
    \end{equation}
    Note that the model joint distribution depends on time $t$ because $f$ is dependent on time.

\subsection{A block-triangular parametrization}
Applying drifting directly with $f_{\theta}$ while using as target the joint law $p_t$ means that  $f_{\theta}$ produces samples $(X(t),X(t+1))$ jointly, but does not in
    general provide a way to fix an arbitrary current state $x(t)$ and sample
    from the transition law $p_t(\cdot\,|\,x(t))$.
To efficiently sample from the transition law, we consider a specific parametrization of $f_{\theta}$, so that when $f_{\theta}$ has learned to sample from the joint law, we additionally obtain a one-step map for the conditional law.
We therefore follow the block-triangular construction of
\cite{baptista2024conditional} and parametrize $f_{\theta}$ as
    \begin{equation}\label{eq:btidform}
        f_\theta(x(t),z,t)
        =
        \left(
        x(t),
        g_\theta(x(t),z,t)
        \right),
    \end{equation}
    where the conditioning variable $x(t)\sim\mu_t$ is a sample from the time marginal $\mu_t$ and $z\sim\pi$ is a reference sample, which is sampled independently from the conditioning variable.
The parametrization \eqref{eq:btidform}  has a  block triangular form because its first component is independent of $z$. In particular, the first marginal induced by it is $\mu_t$ for every value of $\theta$, while only the second component
    $g_\theta$ depends on $\theta$.
The block-triangular structure ensures that matching the joint law recovers the desired conditional law. In particular,
    if
    \[
        f_\theta(\cdot,\cdot,t)_{\sharp}(\mu_t\times\pi)
        =
        p_t,
    \]
    then
    \[
        g_\theta(x(t),\cdot,t)_{\sharp}\pi
        =
        p_t(\cdot\,|\,x(t))
    \]
    for $\mu_t$-almost every $x(t)$; see Theorem~2.4 of
    \cite{baptista2024conditional}.
Thus, by parametrizing $f_{\theta}$ as in  \eqref{eq:btidform}, it is
    sufficient to train the block-triangular map $f_\theta$ to match the joint
    law $p_t$, and one obtains a one-step map for the transition law via $g_{\theta}$.

\subsection{Drifting field}
\label{subsec:ProjectedWFlow}
We now derive a drifting field that is compatible with the joint law as target and with the block-triangular parametrization as model $f_{\theta}$. 
We begin with the drifting field introduced in \cite{han2026wflow} and applying it to the joint law $p_t$ over the state space $\R^{2n}$, 
\begin{equation}\label{eq:WFlowVelocity}
        V_{p_t,q_{\theta}(t)}^{\varepsilon}(y)
        =
        -
        \nabla_y
        \frac{\delta
        \mathcal{S}_{\varepsilon}
        \big(q_{\theta}(t),p_t\big)}
        {\delta q_{\theta}(t)}(y),
        \qquad
        y\in\R^{2n}\,,
    \end{equation}
where $\mathcal{S}_{\varepsilon}$ denotes the Sinkhorn divergence with entropic regularization parameter $\varepsilon>0$; see, e.g., \cite{pmlr-v89-feydy19a}.  The first variation with respect to the model distribution is denoted as $\delta/\delta q_{\theta}(t)$. 

The field in \eqref{eq:WFlowVelocity} acts on both components of the
    joint state. Our block-triangular parametrization \eqref{eq:btidform},
    however, fixes the first component     and therefore cannot realize motion in the dimensions corresponding to the first block of the state space of the joint law.
 With $y=(y_1,y_2)\in\R^n\times\R^n$, we decompose the joint field \eqref{eq:WFlowVelocity} as
    \[
        V_{p_t,q_{\theta}(t)}^{\varepsilon}(y)
        =
        \left(
            V_1^{\varepsilon}(y),
            V_2^{\varepsilon}(y)
        \right),
    \]
    where $V_1^{\varepsilon},V_2^{\varepsilon}:\R^{2n}\to\R^n$ denote joint field's
    first and second components, respectively.
 We then restrict the joint field to the directions compatible with the block-triangular parametrization. Let $P$ denote the orthogonal projection onto the second block. We define
\begin{equation}\label{eq:ProjectedDrift}
        P V_{p_t,q_{\theta}(t)}^{\varepsilon}(y)
        =
        \left(
            0,
            V_2^{\varepsilon}(y)
        \right),
    \end{equation}
    to obtain the projected field, which is the drifting field that we use to update the conditional component $g_{\theta}$.

\subsection{Loss formulation}
\label{subsec:LossFormulation}

For a fixed time $t$, we insert the projected field
    \eqref{eq:ProjectedDrift} into the drifting loss
    \eqref{eq:vanilladriftloss}, which yields
    \begin{align}
        &\mathcal{L}_t(\theta;\theta_j)
        = 
        \\ &\mathbb{E}_{x(t)\sim\mu_t,z\sim\pi}
        \Big[
        \big\|
            f_{\theta}(x(t),z,t)
            -
            \big(
                f_{\theta_j}(x(t),z,t)
                +
                hP V_{p_t,q_{\theta_j}(t)}^{\varepsilon}
                \big(
                    f_{\theta_j}(x(t),z,t)
                \big)
            \big)
        \big\|_2^2
        \Big]. \nonumber
        \label{eq:ProjectedJointLoss}
    \end{align}
Using the block-triangular parametrization
    \eqref{eq:btidform} and the projected field
    \eqref{eq:ProjectedDrift}, the transported target appearing in
    \eqref{eq:ProjectedJointLoss} can be written as
    \begin{multline}
        f_{\theta_j}(x(t),z,t)
        +
        hP V_{p_t,q_{\theta_j}(t)}^{\varepsilon}
        \big(
            f_{\theta_j}(x(t),z,t)
        \big)
        =\\
        \left(
            x(t),\,
            g_{\theta_j}(x(t),z,t)
            +
            hV_2^{\varepsilon}
            \big(
                f_{\theta_j}(x(t),z,t)
            \big)
        \right).
    \end{multline}
Since $f_{\theta}(x(t),z,t)$ has first component $x(t)$ by
    \eqref{eq:btidform}, the first block of the residual in
    \eqref{eq:ProjectedJointLoss} is identically zero. Therefore,
    \eqref{eq:ProjectedJointLoss} reduces to
    \begin{align}
        \mathcal{L}_t(\theta;\theta_j)
        =
        \mathbb{E}_{x(t)\sim\mu_t,z\sim\pi}
        \left[
        \left\|
            g_{\theta}(x(t),z,t)
            -
            \left(
                g_{\theta_j}(x(t),z,t)
                +
                hV_2^{\varepsilon}
                \big(
                    f_{\theta_j}(x(t),z,t)
                \big)
            \right)
        \right\|_2^2
        \right].
        \label{eq:ConditionalDriftLoss}
    \end{align}
Finally, we average the loss over the available time points,
\begin{equation}\label{eq:ProjectedLoss}
        \mathcal{L}(\theta;\theta_j)
        =
        \mathbb{E}_{t\sim\mathcal{U}(\{0,\ldots,T-1\})}
        \left[
\mathcal{L}_t(\theta;\theta_j)
        \right].
    \end{equation}
For each $t$, the drifting field is constructed from the
    corresponding target $p_t$ and current model distribution
    $q_{\theta_j}(t)$, while the parameters $\theta_j$ are shared across time.

\subsection{Theoretical properties}
\label{sec:CDrift:Theory}
Because the joint field \eqref{eq:WFlowVelocity} is the Sinkhorn
    drifting field considered in \cite{han2026wflow}, we can directly build on
    the well-posedness and other theoretical results established therein.
    In particular, the projected field \eqref{eq:ProjectedDrift} inherits the
    regularity conditions required in 
    \cite{han2026wflow}. Indeed, since $P$ is an orthogonal projection,
    \[
        \|PV_{p_t,q}^{\varepsilon}(y) -PV_{p_t,q}^{\varepsilon}(y')\|_2 \leq \|V_{p_t,q}^{\varepsilon}(y)-V_{p_t,q}^{\varepsilon}(y')\|_2
        \qquad
        \text{for all }y,y'\in\R^{2n},
    \]
    so the growth and Lipschitz bounds satisfied by the joint field are
    preserved under projection. Thus, the well-posedness result of
    \cite{han2026wflow} applies also to the projected field.

We additionally note that $q=p_t$ being a (possibly non-unique) fixed point of $PV_{p_t,q}^{\varepsilon}$ follows from the work in \cite{han2026wflow,pmlr-v89-feydy19a} under suitable assumptions. In particular, it is established in \cite{pmlr-v89-feydy19a} that for distributions with compact support, $q=p_t \implies \mathcal{S}_{\varepsilon}(q,p_t) = 0$, and by regularity and non-negativity of the Sinkhorn divergence, $\mathcal{S}_{\varepsilon}(q,p_t) = 0 \implies V^{\varepsilon}_{p_t,q_{\theta}(t)} \equiv 0$. 
Because $P$ is a projection onto the second component of $V^{\varepsilon}_{p_t,q_{\theta}(t)}$,  $p_t$ remains a fixed-point of the projected velocity field.

For drifting, we would like the reverse direction to hold as well, so that we characterize the drift field as
    \[
       PV_{p_t,q}^{0}=0
    \qquad q\text{-a.e.}
    \qquad\Longleftrightarrow\qquad
    q=p_t
    \]
    within the class of joint distributions whose first marginal is $\mu_t$. This ensures that the projection does not introduce additional zeros of the drift field corresponding to joint distributions different from the target distribution.
This implication is not immediate, because projecting a vector field can eliminate nonzero components. Hence it is possible in principle that the projected field vanishes even though the original field does not. We show that this cannot occur for \eqref{eq:ProjectedDrift} when the model distribution $q$ and the target distribution $p_t$ have the same first marginal $\mu_t$ and the distributions are supported compactly and suitably regular. Importantly, the following statement is restricted to the unregularized drift field \eqref{eq:WFlowVelocity}, i.e., $\varepsilon=0$, and thus the following theorem should be viewed as a statement for an idealized setting. The proof proceeds by representing the unregularized drift field through the quadratic-cost optimal transport map from $q$ to $p_t$. We then show that the condition $PV_{p_t,q}^{0}=0$ forces both components of this transport map to coincide with those of the identity map. Consequently, the optimal transport map is the identity $q$-almost everywhere, which implies $q=p_t$.

\begin{theorem}
\label{thm:ProjectedEquilibriumW2}
Set $\varepsilon = 0$ in the drift field \eqref{eq:WFlowVelocity} and fix a time $t$. Let $p_t,q\in\mathcal P_2(\R^{2n})$ have the same first marginal $\mu_t$, where $\mathcal{P}_2(\R^{2n})$ denotes the space of probability measures over $\R^{2n}$ with finite second moment. 
Assume that $q$ and $p_t$ are supported on the closure of $\Omega_1\times\Omega_2$, where $\Omega_1, \Omega_2 \subset \R^{n}$ are bounded, open, convex sets. Assume further that $q$ and $p_t$ are absolutely continuous with respect to Lebesgue measure and that there exist constants $0<\underline{C}\le \overline{C}<\infty$ such that their densities satisfy Lebesgue-almost everywhere
\begin{equation}\label{eq:BoundedDensities}
    \underline{C}\le q\le \overline{C} \quad\text{ on }\Omega_1\times\Omega_2, \qquad \underline{C}\le p_t\le \overline{C} \quad\text{ on }\Omega_1\times\Omega_2.
\end{equation}
Then
\begin{equation}
\label{eq:ProjectedEquilibriumW2}
    PV_{p_t,q}^{0}=0
    \qquad q\text{-a.e.}
    \qquad\Longleftrightarrow\qquad
    q=p_t.
\end{equation}
\end{theorem}

\begin{proof}
Because $q$ is absolutely continuous with respect to the Lebesgue measure,
Brenier's theorem implies that the quadratic-cost optimal transport from
$q$ to $p_t$ is induced by a map
$T_q^{p_t}:\R^{2n}\to\R^{2n}$ that is unique $q$-almost everywhere and of the form
$T_q^{p_t}=\nabla u$, where $u:\R^{2n}\to\R$ is convex \cite[Theorem~2.26]{Ambrosio2013}. By assumption \eqref{eq:BoundedDensities}, the densities $q$ and $p_t$ are bounded above and bounded away from zero on $\Omega_1 \times \Omega_2$. Therefore, by the regularity theorem for quadratic optimal transport \cite[Theorem~2.27]{Ambrosio2013}, the optimal transport map $T_q^{p_t}$ admits a H\"older-continuous, and hence continuous, representative on $\Omega_1\times\Omega_2$. We use this continuous representative in the following. Furthermore, since $u$ is convex and $\nabla u=T_q^{p_t}$ almost everywhere on $\Omega_1\times\Omega_2$, the continuity of this representative implies that $u$ is continuously differentiable and $\nabla u = T_{q}^{p_t}$ everywhere on $\Omega_1 \times \Omega_2$.

We now relate $T_q^{p_t}$ to the unregularized drifting field used in
\eqref{eq:WFlowVelocity}. For the quadratic cost
$c(y,\bar y)=\frac12\|y-\bar y\|_2^2$, a source Kantorovich potential
associated with the Brenier potential $u$ is
\[
    \Phi_q(y)
    =
    \frac12\|y\|_2^2-u(y),
\]
up to an additive constant; see \cite[Proposition~1.21]{santambrogio2015optimal} and the discussion following that proposition. Now note that 
the lower bound
\eqref{eq:BoundedDensities} implies that the support of $q$ is the closure $\overline{\Omega_1 \times \Omega_2}$, and likewise for $p_t$. Additionally, the quadratic cost belongs
to $C^1(\overline{\Omega_1 \times \Omega_2} \times \overline{\Omega_1 \times \Omega_2})$. Hence the assumptions of
\cite[Proposition~7.18]{santambrogio2015optimal} are satisfied, which shows that the
Kantorovich potential is unique up to an additive constant.

Now recall that for $\varepsilon = 0$, the Sinkhorn divergence $\mathcal{S}_{\varepsilon}$ used in the definition of the drift field \eqref{eq:WFlowVelocity} becomes $\mathcal{S}_0(q, p_t) = \frac{1}{2}W_2^2(q, p_t)$; see, e.g., \cite{han2026wflow}. By \cite[Proposition~7.17]{santambrogio2015optimal}, the Kantorovich potential is a subgradient of the functional $q\mapsto\frac12W_2^2(q,p_t)$ and when it is unique up to additive constants, it represents its first variation, 
\[
\frac{\delta\mathcal S_0(q,p_t)}{\delta q}
=
\Phi_q
\]
up to an additive constant. Consequently, for $q$-almost every
$y\in\Omega_1\times\Omega_2$,
\[
    V_{p_t,q}^0(y)
    =
    -\nabla\Phi_q(y)
    =
    \nabla u(y)-y
    =
    T_q^{p_t}(y)-y.
\]
Hence the unregularized drifting field is precisely the displacement
field of the quadratic optimal transport from $q$ to $p_t$.

Now we are ready to prove the implication
\[
    q=p_t
    \quad\Longrightarrow\quad
    PV_{p_t,q}^0=0
    \quad q\text{-almost everywhere}\,.
\]
Assume that $q=p_t$. Since the identity map transports $q$ to itself
with zero quadratic cost, $W_2(q,q)=0$. Because $T_q^q$ is an optimal
transport map from $q$ to itself,
\[
    \frac12
    \int_{\R^{2n}}
        \|T_q^q(y)-y\|_2^2
        \,q(\mathrm dy)
    =
    \frac12W_2^2(q,q)
    =
    0.
\]
The integrand is nonnegative, and therefore
$T_q^q(y)=y$ for $q$-almost every $y$. It follows from
$V_{q,q}^0=T_q^q-\operatorname{Id}$ that
$V_{q,q}^0=0$ $q$-almost everywhere, and hence
$PV_{q,q}^0=0$ $q$-almost everywhere.

We now prove the converse. Assume that
$PV_{p_t,q}^0=0$ $q$-almost everywhere. The strategy is now to show that under this assumption, the map $T_q^{p_t}$ must be the identity map. We do this component-wise, starting with the second component. Write
$y=(y_1,y_2)\in\R^n\times\R^n$ and decompose the optimal transport map
as
\[
    T_q^{p_t}(y_1,y_2)
    =
    \bigl(
        T_1(y_1,y_2),
        T_2(y_1,y_2)
    \bigr).
\]
Recalling that $P(v_1,v_2)=(0,v_2)$, the identity
$V_{p_t,q}^0=T_q^{p_t}-\operatorname{Id}$ gives
\[
    PV_{p_t,q}^0(y_1,y_2)
    =
    \bigl(
        0,
        T_2(y_1,y_2)-y_2
    \bigr).
\]
Consequently,
\begin{equation}\label{eq:Proof:IdentityT2}
    T_2(y_1,y_2)=y_2
    \qquad
    q\text{-a.e. }(y_1,y_2) \in \Omega_1 \times \Omega_2.
\end{equation}

We next show that \eqref{eq:Proof:IdentityT2} holds everywhere on
$\Omega_1\times\Omega_2$. By \eqref{eq:BoundedDensities},
$q(y)\geq \underline{C} >0$ for Lebesgue-almost every $y \in \Omega_1\times\Omega_2$. Hence, for every measurable
set $A\subset\Omega_1\times\Omega_2$,
\[
    q(A)
    =
    \int_A q(y)\,\mathrm dy
    \geq
    \underline{C}\,|A|,
\]
where $|A|$ denotes the Lebesgue measure of $A$. Consequently,
$q(A)=0$ implies $|A|=0$. Thus, \eqref{eq:Proof:IdentityT2} implies
\[
    T_2(y_1,y_2)=y_2
    \qquad
    \text{for Lebesgue-almost every }(y_1,y_2)
    \in\Omega_1\times\Omega_2.
\]
Since $T_q^{p_t}$ is continuous on $\Omega_1\times\Omega_2$, the map
$(y_1,y_2)\mapsto T_2(y_1,y_2)-y_2$ is continuous there. A continuous
function that vanishes Lebesgue-almost everywhere on an open set must
vanish everywhere on that set. Therefore
\begin{equation}\label{eq:Proof:T2IsIdentity}
    T_2(y_1,y_2)=y_2
    \qquad
    \text{for every }(y_1,y_2)\in\Omega_1\times\Omega_2.
\end{equation}

Let us now consider the first component of $T_q^{p_t}$ and show that it agrees with the identity $q$-almost everywhere. Fix an arbitrary reference point $y_2^0\in\Omega_2$ and define
$h:\Omega_1\to\R$ by
\[
    h(y_1)
    =
    u(y_1,y_2^0)
    -
    \frac12\|y_2^0\|_2^2.
\]
Fix arbitrary $y_1\in\Omega_1$ and $y_2\in\Omega_2$. Since $\Omega_2$
is convex, the line segment
\[
    y_2(r)=(1-r)y_2^0+ry_2,
    \qquad r\in[0,1],
\]
is contained in $\Omega_2$. Define
$\gamma(r)=u(y_1,y_2(r))$. Since $u$ is continuously differentiable, the chain rule gives
\begin{align*}
    \gamma'(r)
    &=
    \nabla_{y_2}u(y_1,y_2(r))
    \cdot
    (y_2-y_2^0) =
    y_2(r)\cdot(y_2-y_2^0),
\end{align*}
where we used
$\nabla_{y_2}u(y_1,y_2)=y_2$ on
$\Omega_1\times\Omega_2$ which follows from \eqref{eq:Proof:T2IsIdentity} with $T_{q}^{p_t} = \nabla u$. Therefore,
\begin{align*}
    u(y_1,y_2)-u(y_1,y_2^0)
    &= \int_0^1 y_2(r)\cdot(y_2-y_2^0)\,\mathrm dr \\
    &= \int_0^1 \bigl((1-r)y_2^0+ry_2\bigr)\cdot(y_2-y_2^0)\,\mathrm dr \\
    &= y_2^0\cdot(y_2-y_2^0)\int_0^1(1-r)\,\mathrm dr
       + y_2\cdot(y_2-y_2^0)\int_0^1r\,\mathrm dr \\
    &= \frac12 y_2^0\cdot(y_2-y_2^0)
       + \frac12 y_2\cdot(y_2-y_2^0) \\
    &= \frac12 (y_2+y_2^0)\cdot(y_2-y_2^0)
     = \frac12\|y_2\|_2^2-\frac12\|y_2^0\|_2^2.
\end{align*}
Hence
\begin{equation}\label{eq:IdentityUhy}
    u(y_1,y_2)
    =
    h(y_1)
    +
    \frac12\|y_2\|_2^2
\end{equation}
for every $(y_1,y_2)\in\Omega_1\times\Omega_2$.
Since $u$ is convex, the restriction $y_1\mapsto u(y_1,y_2^0)$ is convex, and subtracting the constant $\frac12\|y_2^0\|_2^2$ shows that $h$ is convex as well.
Since $u$ is continuously differentiable, $h$ is also continuously differentiable. Differentiating \eqref{eq:IdentityUhy} therefore gives
\[
    T_q^{p_t}(y_1,y_2)
    =
    \nabla u(y_1,y_2)
    =
    \bigl(
        \nabla h(y_1),
        y_2
    \bigr),
\]
everywhere on $\Omega_1 \times \Omega_2$. 

It remains to determine the first component $\nabla h(y_1)$. Let
$P_1:\R^{2n}\to\R^n$ denote the projection
$P_1(y_1,y_2)=y_1$. By assumption, $q$ and $p_t$ have the same first
marginal, so
\[
    (P_1)_\sharp q
    =
    \mu_t
    =
    (P_1)_\sharp p_t.
\]
Since $(T_q^{p_t})_\sharp q=p_t$, we obtain
\begin{align*}
    \mu_t
    &=
    (P_1)_\sharp p_t= (P_1)_\sharp\bigl((T_q^{p_t})_\sharp q\bigr)=
    (P_1\circ T_q^{p_t})_\sharp q\\
    &=
    (\nabla h\circ P_1)_\sharp q=
    (\nabla h)_\sharp\bigl((P_1)_\sharp q\bigr)=
    (\nabla h)_\sharp\mu_t,
\end{align*}
which shows that $\nabla h$ transports $\mu_t$ to itself. We already established that  $h(y_1)=u(y_1,y_2^0)-\frac12\|y_2^0\|_2^2$ is convex. Thus, the map $\nabla h$ is an optimal transport map for the quadratic cost \cite[Theorem~2.13]{Ambrosio2013}. The identity map also transports $\mu_t$ to itself and has zero quadratic cost. Hence the optimal transport cost from $\mu_t$ to itself is zero. Since $\nabla h$ is optimal, it must also attain zero cost. Because the quadratic cost is nonnegative and vanishes only when source and target points coincide, it follows that
\[
    \nabla h(y_1)=y_1
    \qquad
    \mu_t\text{-almost everywhere}.
\]

Since $(P_1)_\sharp q=\mu_t$, the identity $\nabla h(y_1)=y_1$ $\mu_t$-almost everywhere implies, together with the representation $T_q^{p_t}(y_1,y_2)=(\nabla h(y_1),y_2)$, that
\[
    T_q^{p_t}(y_1,y_2) = (y_1,y_2)
\]
for $q$-almost every $(y_1,y_2)\in\Omega_1\times\Omega_2$. Thus $T_q^{p_t}=\operatorname{Id}$ $q$-almost everywhere. Because  $(T_q^{p_t})_\sharp q=p_t$, it follows that
\[
p_t = (T_q^{p_t})_\sharp q = (\operatorname{Id})_\sharp q= q.
\]
This proves the converse implication and hence \eqref{eq:ProjectedEquilibriumW2}.
\end{proof}

\subsection{Empirical loss and computational procedure}
\label{subsec:EmpiricalLoss}
We train with the trajectory data \eqref{eq:data} by replacing the
    distributions in \eqref{eq:ProjectedLoss} with empirical distributions
    constructed from mini-batches. At each training iteration $j$, we sample
    a set
    \[
        \mathcal{T}
        =
        \{t_1,\ldots,t_{n_t}\}
        \subset
        \{0,\ldots,T-1\}
    \]
    of $n_t$ time indices uniformly. 
For each $t\in\mathcal{T}$, we draw $n_B$ transition pairs \\ $\{(x^b(t),x^b(t + 1))\}_{b=1}^{n_B}$
    from the trajectories in \eqref{eq:data}. These samples define the
    empirical target distribution
    \begin{equation}\label{eq:EmpiricalTarget}
        \widehat p_t
        =
        \frac{1}{n_B}
        \sum_{b=1}^{n_B}
        \delta_{(x^b(t),x^b(t+1))}.
    \end{equation}
To approximate the current joint model distribution
    \eqref{eq:TargetJointatTheta}, we independently sample conditioning
    states $\{\bar x^b(t)\}_{b=1}^{n_B}
        \sim \mu_t$ from the observed states at time $t$ and reference samples
    $\{z^b\}_{b=1}^{n_B}\sim\pi$. We then compute
    \[
        y_{j}^b(t)
        =
        f_{\theta_j}(\bar x^b(t),z^b,t),
        \qquad b=1,\ldots,n_B,
    \]
    and define
    \begin{equation}\label{eq:EmpiricalModel}
        \widehat q_{\theta_j}(t)
        =
        \frac{1}{n_B}
        \sum_{b=1}^{n_B}
        \delta_{y_{j}^b(t)}.
    \end{equation}
We further draw an independent second batch
    $\{\bar x'^{\,b}(t)\}_{b=1}^{n_B}\sim\mu_t$ and
    $\{z'^{\,b}\}_{b=1}^{n_B}\sim\pi$, and form
    \[
        y_{j}'^{\,b}(t)
        =
        f_{\theta_j}(\bar x'^{\,b}(t),z'^{\,b},t),
        \qquad
        \widehat q_{\theta_j}'(t)
        =
        \frac{1}{n_B}
        \sum_{b=1}^{n_B}
        \delta_{y_{j}'^{\,b}(t)}.
    \]
For each $t\in\mathcal{T}$, we compute the empirical joint field $\widehat{V}_{t,j}^{\varepsilon}$ from $\widehat q_{\theta_j}(t), \widehat p_t$ and the independent empirical
model distribution $\widehat q_{\theta_j}'(t)$ using the Sinkhorn barycentric projections of
    \cite{han2026wflow}. We then project this field according to
    \eqref{eq:ProjectedDrift},
    \[
        P\widehat V_{t,j}^{\varepsilon}(y)
        =
        \left(
            0,
            \widehat V_{2,t,j}^{\varepsilon}(y)
        \right).
    \]
The empirical counterpart of
    \eqref{eq:ConditionalDriftLoss} is therefore
    \begin{align}
        \widehat{\mathcal L}(\theta;\theta_j)
        &=
        \frac{1}{n_t n_B}
        \sum_{t\in\mathcal T}
        \sum_{b=1}^{n_B}
        \Bigg\|
            g_\theta(\bar x^b(t),z^b,t)
            -
               \left( g_{\theta_j}(\bar x^b(t),z^b,t)
                +
                h\widehat V_{2,t,j}^{\varepsilon}
                (y_{j}^b(t))\right)
        \Bigg\|_2^2.
        \label{eq:EmpiricalProjectedLoss}
    \end{align}
     A gradient step on
    \eqref{eq:EmpiricalProjectedLoss} then updates $\theta_j$ to
    $\theta_{j+1}$.
After training, given an initial state $\widehat X(0)\sim\mu_0$, a sample
    trajectory is generated autoregressively by
\begin{equation}\label{eq:AutoregressiveSampling}
        z(t)\sim\pi,
        \qquad
        \widehat X(t+1)
        =
        g_{\theta}
        \big(
            \widehat X(t),z(t),t
        \big),
        \qquad
        t=0,\ldots,T-1.
    \end{equation}
    Thus, the block-triangular map $f_\theta$ is used only implicitly during training to
    construct the empirical joint drifting field, while the second component
    $g_\theta$ of $f_{\theta}$ is the one-step transition map used at inference time.

\newcommand{\pmc}[2]{#1{(\pm#2)}}

\begin{table*}[t]
\centering
\caption{Duffing oscillator: BTJD achieves the lowest reported marginal and trajectory-QoI errors for both random and fixed initial conditions.}
\label{tab:duffing}
\setlength{\tabcolsep}{2.5pt}
\renewcommand{\arraystretch}{0.92}
\resizebox{\textwidth}{!}{%
\begin{tabular}{lllll}
\toprule
 & \multicolumn{2}{c}{Duffing (IC: $X(0) \sim \mathcal{N}([0, -10]^{\top}, I_2)$)}
 & \multicolumn{2}{c}{Duffing (IC: $X(0) = [0, -10]^{\top}$)} \\
\cmidrule(lr){2-3} \cmidrule(lr){4-5}
 & dist.\ err.\ ($W_2$) & traj.\ QoI err.\
 & dist.\ err.\ ($W_2$) & traj.\ QoI err.\ \\
\midrule
T.\ stepper \cite{otness2021an}
& $\pmc{3.48\mathrm{e}{-01}}{2.00\mathrm{e}{-01}}$
& $\pmc{3.68\mathrm{e}{-02}}{3.30\mathrm{e}{-04}}$
& -- & -- \\
DICE \cite{blickhan-berman-stuart-etal:2025}
& $\pmc{5.80\mathrm{e}{-01}}{5.00\mathrm{e}{-01}}$
& $\pmc{8.30\mathrm{e}{-02}}{5.80\mathrm{e}{-04}}$
& -- & -- \\
Marginal diff.\ \cite{ho2020denoising}
& $\pmc{8.10\mathrm{e}{-02}}{2.50\mathrm{e}{-02}}$
& $\pmc{1.07\mathrm{e}{-01}}{5.00\mathrm{e}{-04}}$
& -- & -- \\
SDE learning \cite{dridi2021learning}
& $\pmc{8.30\mathrm{e}{-02}}{2.80\mathrm{e}{-02}}$
& $\pmc{5.70\mathrm{e}{-03}}{4.50\mathrm{e}{-04}}$
& $\pmc{1.08\mathrm{e}{-01}}{8.00\mathrm{e}{-02}}$
& $\pmc{6.80\mathrm{e}{-03}}{4.50\mathrm{e}{-04}}$ \\
SDE matching \cite{bartosh2025sdematching}
& $\pmc{5.49\mathrm{e}{-01}}{2.00\mathrm{e}{-01}}$
& $\pmc{4.47\mathrm{e}{-02}}{6.20\mathrm{e}{-04}}$
& $\pmc{3.94\mathrm{e}{-01}}{2.20\mathrm{e}{-01}}$
& $\pmc{9.07\mathrm{e}{-02}}{5.60\mathrm{e}{-04}}$ \\
BTJD (ours)
& $\pmc{\mathbf{4.80\mathrm{e}{-02}}}{1.60\mathrm{e}{-02}}$
& $\pmc{\mathbf{2.75\mathrm{e}{-03}}}{4.21\mathrm{e}{-04}}$
& $\pmc{\mathbf{5.80\mathrm{e}{-02}}}{2.20\mathrm{e}{-02}}$
& $\pmc{\mathbf{2.49\mathrm{e}{-03}}}{4.24\mathrm{e}{-04}}$ \\
\bottomrule
\end{tabular}%
}
\end{table*}

\begin{table*}[t]
\centering\small
\caption{Rayleigh-B\'{e}nard convection: BTJD achieves the lowest reported marginal and rotational-current errors at the unseen Rayleigh parameter.}
\label{tab:rb_convection}
\begin{tabular}{lll}
\toprule
 & dist.\ err.\ ($W_2$) & traj.\ QoI err.\ \\
\midrule
T.\ stepper \cite{otness2021an}
& $\pmc{2.69\mathrm{e}{+00}}{1.60\mathrm{e}{+00}}$
& $\pmc{5.80\mathrm{e}{-03}}{1.46\mathrm{e}{-03}}$ \\
DICE \cite{blickhan-berman-stuart-etal:2025}
& $\pmc{2.00\mathrm{e}{-01}}{1.20\mathrm{e}{-01}}$
& $\pmc{9.30\mathrm{e}{-02}}{1.46\mathrm{e}{-03}}$ \\
Marginal diffusion \cite{ho2020denoising}
& $\pmc{1.11\mathrm{e}{-01}}{6.10\mathrm{e}{-02}}$
& $\pmc{6.10\mathrm{e}{-01}}{1.15\mathrm{e}{-02}}$ \\
SDE learning \cite{dridi2021learning}
& $\pmc{5.40\mathrm{e}{-02}}{2.30\mathrm{e}{-02}}$
& $\pmc{2.20\mathrm{e}{-02}}{1.66\mathrm{e}{-03}}$ \\
SDE matching \cite{bartosh2025sdematching}
& $\pmc{2.27\mathrm{e}{-01}}{1.30\mathrm{e}{-01}}$
& $\pmc{1.90\mathrm{e}{-02}}{1.67\mathrm{e}{-03}}$ \\
BTJD (ours)
& $\pmc{\mathbf{4.60\mathrm{e}{-02}}}{3.10\mathrm{e}{-02}}$
& $\pmc{\mathbf{3.00\mathrm{e}{-04}}}{1.36\mathrm{e}{-03}}$ \\
\bottomrule
\end{tabular}%
\end{table*}

\section{Experiments}
To demonstrate the performance of our algorithm for stochastic transition modeling, we assess our approach on four problems and a range of baselines.

\subsection{Baselines}
\label{sec:baselines}

We compare BTJD with deterministic surrogate models, marginal-matching methods,
learned stochastic differential equations, multi-step conditional generative models,
and one- or few-step distilled generative models. Baseline
implementations and training setups follow \cite{jha2026firstordertrajectorymatchingfast}.

\emph{Deterministic surrogate models.}
We compare against deterministic surrogate models that learn a single successor state from the current state. We consider a learned deterministic time stepper \cite{otness2021an} for the low-dimensional problems and a field-to-field surrogate akin to operator learning \cite{stachenfeld2021learned} for the PDE problems. Because these models return a single successor for a given state, they cannot represent the intrinsic stochastic variability of the dynamics.

\emph{Marginal-matching methods.}
We compare against methods that learn the evolution of the time marginals without identifying the stochastic transition law between consecutive states. These include DICE \cite{blickhan-berman-stuart-etal:2025}, which learns deterministic population
dynamics consistent with the observed marginals, and a diffusion-based marginal matching approach \cite{ho2020denoising}, which trains a time-conditioned
diffusion model to generate samples from the marginal distribution $\mu_t$ at each
physical time $t$, without conditioning on the preceding state. Thus, such models can sample from time-marginal distributions at individual times, but the marginals alone do not determine trajectory-dependent statistics.

\emph{Learned stochastic differential equations}
Stochastic surrogate models based on learned SDEs explicitly represent random state evolution through learned drift and diffusion terms. One approach fits these coefficients from consecutive trajectory observations using an Euler--Maruyama transition model \cite{dridi2021learning}, while SDE Matching \cite{bartosh2025sdematching} learns a generative SDE using a simulation-free matching objective. In both cases, trajectories are obtained by simulating the learned stochastic dynamics.

\emph{Conditional diffusion and flow models.}
Conditional generative models directly target the transition law, but typically require an auxiliary sampling procedure at every physical time step. We consider autoregressive diffusion models (ARDM), following
\cite{KOHL2026108641} and building on denoising diffusion probabilistic models \cite{ho2020denoising}, which generate each successor through reverse diffusion, and conditional flow matching (CFM) \cite{albergo-vanden-eijnden:2022,lipman2023flow}, which generates successors by integrating a learned conditional flow. Their inference cost therefore grows with the number of denoising or flow-integration steps used per transition.

\emph{One- and few-step generative models.}
One- and few-step generative methods reduce the inference cost of conditional generative models after training. We focus on MeanFlow-based distillation
\cite{geng2025mean} and use it to compress a pretrained
conditional flow into an average-velocity model, while ReFlow \cite{liu2023rectified} progressively straightens the generative flow and then compresses it into a one-step sampler. We note that in contrast BTJD directly learns the one-step transition map without first training and compressing a multi-step conditional generator.

\subsection{Duffing oscillator} We first consider a duffing oscillator with stochastic forcing. 
\subsubsection{Duffing oscillator: Setup and training data} The duffing oscillator that we consider is governed by 
    \begin{align}\label{eq:duffing}
        &dX_1(\tau) = X_2(\tau) d\tau, \\
        &dX_2(\tau) = (-2 \xi \omega X_2(\tau) + \omega^2 X_1(\tau) - \omega^2 \gamma X_1(\tau)^3)d\tau + \sigma dW(\tau)\,,
    \end{align}
    where the first equation determines position and the second equation describes the dynamics of the velocity. The variable $\xi$ is a damping parameter, $\gamma$ determines the strength of the cubic term, and $\omega$ controls the stiffness of the linear dynamics. The strength of the Brownian motion is set by $\sigma$. We set these parameters to
    \begin{equation}\label{eq:DuffingHPs}
        \xi = 0.2, \quad \gamma=0.2, \quad \omega = 1, \quad\sigma=0.5\,.
    \end{equation}

The SDE \eqref{eq:duffing} is numerically integrated with the Euler-Maruyama scheme with step size $\Delta \tau = 0.01$ on the time interval $[0,\tau_{\textrm{end}}] = [0,12]$, yielding $T = 1200$ time steps.
The training data \eqref{eq:data} are generated with initial conditions sampled from $\mathcal{N}([0, -10]^{\top}, I_2)$, where $I_2 \in \mathbb{R}^{2 \times 2}$ is the identity matrix. We generate $N = 5000$ training trajectories.
The block-triangular map is parametrized as an MLP. The MLP is a 2-hidden layer MLP with width 512 per layer and SiLU activations. Time is concatenated as a scalar to the input.
The $\varepsilon$ in the Sinkhorn loss is set to $0.01$, $h=0.1$, and $50$ Sinkhorn iterations are performed per step.
The model trains by sampling $512 \times 4=2048$ particles per step, $512$ particles from $4$ independent timesteps, for $100$k gradient steps with AdamW at a learning rate of $1\mathrm{e}{-3}$.

\subsubsection{Duffing oscillator: Test initial conditions}\label{sec:NumExp:Duffing:TestInit}
To generate test data, we consider two different initial conditions.  For the first test data set, we draw initial conditions from $\mathcal{N}([0, -10]^{\top}, I_2)$ as well, which is meant to assess the approach's ability to generalize to unseen initial conditions. For the second test data set, we have a deterministic (fixed) initial condition $X(0) = [0, -10]^{\top}$. Having a deterministic initial condition helps to assess how well the approach generates different paths from the same initial condition. 

We demonstrate the performance of our approach based on two error measures. First, we compute the sliced Wasserstein-2 distance between 5000 generated samples and 5000 test-set samples at each timestep $t$, started from the same initial condition. This yields 1200 sliced Wasserstein-2 distances, one for each time step, which we then average and report, along with the standard deviation of this distance over timesteps. This metric is meant to assess agreement between the true marginal distribution and marginal distribution predicted by our method.
Second, we compute an error measure of a trajectory-dependent quantity in order to assess whether the model captures the trajectory-dependent dynamics of the system. We consider the quantity of interest (QoI) \begin{equation}\label{eq:StratQoI}
        Q_{\phi} = \mathbb{E} \left[\int_0^{\tau_{\textrm{end}}} \phi(\tau,X(\tau)) \circ dX(\tau) \right]
    \end{equation}
    for the smooth test function 
    \begin{equation}\label{eq:DuffingQoI}
        \phi(x) = \begin{bmatrix}
            \frac{1}{\sqrt{2\pi}}\exp(\frac{-x_1^2}{2}) \tanh(x_2) \\
            0
        \end{bmatrix}\,.
    \end{equation}
    The quantity describes a velocity-weighted, smoothed counting of the crossings over the barrier $x_1=0$ of each sample $X^{(i)}$ over the time interval. 
We estimate \eqref{eq:StratQoI} from samples via 
\begin{align}\label{eq:StratMidpoint}\small
        & \hat{q}_{\phi}(X^{(i)}) =  \sum_{k=0}^{K-1} \phi \left(\frac{t_k+t_{k+1}}{2},\frac{X^{(i)}(t_k)+X^{(i)}(t_{k+1})}{2} \right)^\top \left(X^{(i)}(t_{k+1}) -  X^{(i)}(t_k)\right)\,, \\
        & \hat{Q}_{\phi}(\mathcal{D}) = \frac{1}{N}\sum_{i=1}^N \hat{q}_{\phi}(X^{(i)})\, \nonumber.
    \end{align}
Due to the variance of this quantity across trajectories, we evaluate the difference between this quantity in 200,000 generated and test-set trajectories from the same initial condition. We report the mean relative error of \eqref{eq:StratMidpoint} obtained with generated trajectories versus ground-truth trajectories as well as its standard error, i.e. for ground truth realizations $X$ and generated trajectories  $\hat{X}$:
\begin{equation}\label{eq:StratRelErrSE}
    \frac{\left(\textrm{Var}\left[ \hat{q}_{\phi}(X)- \hat{q}_{\phi}(\hat{X}) \right]\right)^{\frac{1}{2}}}{\sqrt{N} \left| \hat{Q}_{\phi}(\mathcal{D}) \right|}\,.
\end{equation}

\subsubsection{Duffing oscillator: Results}\label{sec:NumExp:Duffing:Results} \Cref{fig:Duffing:TruthVsLiftDrift} compares ground-truth trajectories with BTJD rollouts from the same deterministic initial condition. Despite starting from a single state, BTJD generates a diverse ensemble with the qualitative variability of the ground-truth trajectories. This demonstrates that the model captures stochasticity in the transition dynamics rather than relying on variability in the initial condition.

Quantitatively, \Cref{tab:duffing} shows that BTJD achieves the lowest sliced-$W_2$ error among the tested methods for both random and deterministic test initial conditions. Thus, the improved agreement of the time marginals persists both for unseen initial conditions drawn from the training distribution and when all trajectories start from the same state. The phase-space snapshots in \Cref{fig:Duffing:Marginal} further show agreement between the generated and ground-truth marginal distributions throughout the time integration. BTJD also achieves the lowest error in the trajectory-dependent QoI \eqref{eq:StratQoI} for both test initial condition distributions; see \Cref{tab:duffing}. Hence, the improved performance achieved by BTJD is not limited to matching time marginals; BTJD also predicts statistics that depend on the temporal evolution of individual trajectories.

\begin{figure}[tb]
  \centering
 \begin{tabular}{cc}
 \hspace*{-0.5cm}\includegraphics[width=0.48\textwidth]{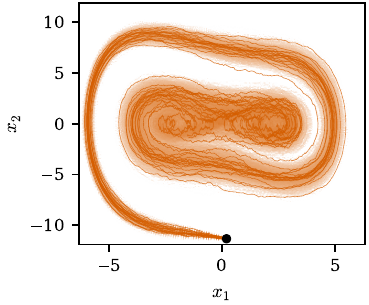} & \includegraphics[width=0.48\textwidth]{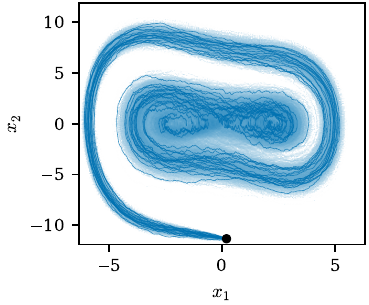}\\
 (a) ground truth & (b) BTJD (ours)
 \end{tabular}
  \caption{Duffing oscillator: From the same deterministic initial state, BTJD generates diverse trajectories that capture well the stochastic variability of the duffing oscillator.}
  \label{fig:Duffing:TruthVsLiftDrift}
\end{figure}
\begin{figure}[tb]
  \centering
  \begin{tabular}{cccc}\hspace*{-0.8cm}
    \includegraphics[width=0.25\textwidth]{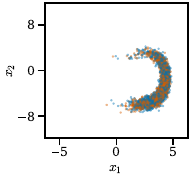} &\hspace*{-0.5cm}
    \includegraphics[width=0.25\textwidth]{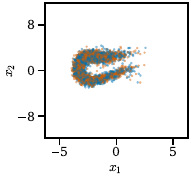} &\hspace*{-0.5cm}
    \includegraphics[width=0.25\textwidth]{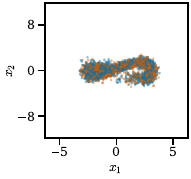} & \hspace*{-0.5cm}
    \includegraphics[width=0.25\textwidth]{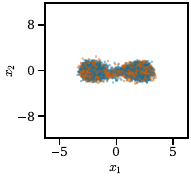} \\
    (a) $\tau = 3$ & (b) $\tau = 6$ & (c) $\tau = 9$ & (d) $\tau = 12$
  \end{tabular}
  \caption{Duffing oscillator: BTJD tracks the evolving phase-space distribution throughout the rollout; ground-truth particles (\textcolor[HTML]{D55E00}{orange}) and BTJD (ours) particles (\textcolor[HTML]{0072B2}{blue}) at four rollout times.}
  \label{fig:Duffing:Marginal}
\end{figure}

\subsection{Rayleigh-Bénard convection} We now consider the 9-dimensional Rayleigh-Bénard convection model  \cite{reiterer-lainscek-schuerrer-etal:1998} with an additive stochastic forcing. 
\subsubsection{Rayleigh-Bénard: Setup} 
The process $$X(\tau) = (X_1(\tau),X_2(\tau),\dots,X_9(\tau)) \in \R^9$$ is governed by
\begin{equation}\label{eq:RBConvectionDef}
        dX(\tau) = f_{\mu}(X(\tau))d\tau+\sigma dW(\tau)
    \end{equation}
    where $f_{\mu}$ is the deterministic vector field of \cite{reiterer-lainscek-schuerrer-etal:1998}, determined by a Rayleigh-type parameter $\mu$. We set  $\sigma = 0.05$.
For training initial conditions, we draw $20000$ samples from $\mathcal{N}(0,0.0004\cdot I_9)$ and generate the corresponding trajectories via Euler-Maruyama with $\Delta \tau = 0.01$ on the interval $[0,20]$ for each control parameter $\mu \in \{ 13.5,13.6,13.7,\dots,14.2 \}$. 

We generate a test data set by drawing samples from the same distribution of initial conditions. 
   For the test data set, the Rayleigh control parameter $\mu$ is set to $\mu = 13.65$ to evaluate the model's ability to generalize to unseen parameters. Analogous to \Cref{sec:NumExp:Duffing:TestInit}, we compute sliced Wasserstein-2 distance to assess the time marginals, reporting the mean distance and its standard error across timesteps.
Additionally, we compute the relative absolute errors of estimating the quantity \eqref{eq:StratQoI} but with test function \begin{equation}\label{eq:RBQoI}
        \phi(x) = Q^\top RQ x, \quad Qx = (x_6,x_7), \quad R = \begin{bmatrix}
            0 & -1 \\
            1 & 0
        \end{bmatrix}\,.
    \end{equation}
    The quantity measures the portion of probability mass motion which is aligned with counterclockwise rotation in the $(x_6,x_7)$ plane. 
    The system has a  probability current with persistent rotation in low-dimensional projections, making this a well-defined quantity to assess trajectory-dependent dynamics. As in \Cref{sec:NumExp:Duffing:TestInit}, we report the mean relative error across trajectories as well as the standard error over trajectories given by \eqref{eq:StratRelErrSE}.

We train a $3$-layer MLP with width of $512$ neurons per layer. Time is embedded as a sinusoidal embedding with $50$ frequencies and then concatenated to the input, along with $\mu$. For this experiment, we train with $\varepsilon = 0.1$, $h=0.1$, $20$ Sinkhorn iterations per step. Each gradient step samples $1024 \times 8 = 8192$ particles, $1024$ particles per timestep and control parameter, with $4$ timesteps sampled for $2$ control parameters in each gradient step. We train for $100$k gradient steps using AdamW at a learning rate of $1\mathrm{e}{-3}$.

\subsubsection{Rayleigh-Bénard: Results} 
We evaluate the BTJD model at the unseen control parameter $\mu=13.65$, which lies between parameter values used during training. We plot in \Cref{fig:RB:Marginal} histograms of the marginals corresponding to dimension $1, 2, 4, 8$, and $9$. Our BTJD is accurately approximating the ground-truth marginal distribution. As reported in \Cref{tab:rb_convection}, BTJD achieves the lowest sliced-$W_2$ error among all methods. The accuracy of the predicted time marginals therefore extends to interpolation in the Rayleigh control parameter. BTJD also achieves the lowest error in the rotational-current QoI \eqref{eq:RBQoI}; see \Cref{tab:rb_convection}. This quantity depends on the direction of probability-mass motion in the $(x_6,x_7)$ plane, so the result shows that BTJD captures trajectory-dependent dynamics that are not determined by the time marginals alone. 

\begin{figure}[tb]
  \centering
  \setlength{\tabcolsep}{0pt}
  \newcommand{\rbpanel}[1]{%
    \raisebox{-0.5\height}{\includegraphics[width=0.165\textwidth]{figures/rb_histograms/#1}}}
  \newcommand{\rblabel}[1]{%
    \raisebox{-0.5\height}{\rotatebox[origin=c]{45}{%
      \shortstack{\small marginal\\[1pt] \small in  $#1$}}}}
  \begin{tabular}{@{}c@{\;}c@{\,}c@{\,}c@{\,}c@{\,}c@{}}
    & $\tau = 0$ & $\tau = 5$ & $\tau = 10$ & $\tau = 15$ & $\tau = 20$ \\[2pt]
    \rblabel{X_1(\tau)} &
    \rbpanel{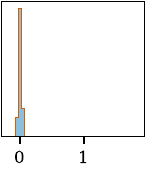} &
    \rbpanel{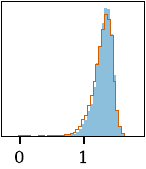} &
    \rbpanel{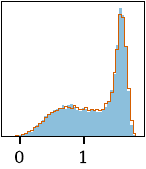} &
    \rbpanel{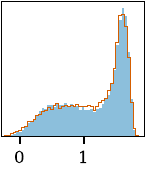} &
    \rbpanel{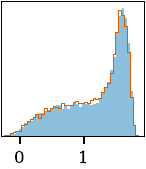} \\
    \rblabel{X_2(\tau)} &
    \rbpanel{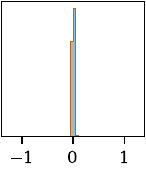} &
    \rbpanel{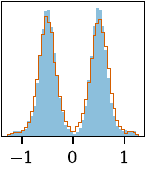} &
    \rbpanel{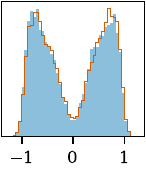} &
    \rbpanel{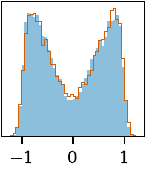} &
    \rbpanel{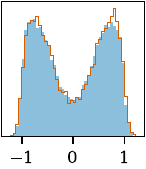} \\
    \rblabel{X_4(\tau)} &
    \rbpanel{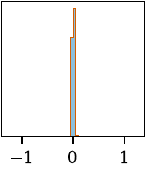} &
    \rbpanel{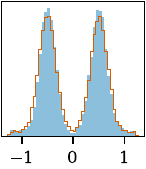} &
    \rbpanel{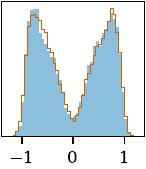} &
    \rbpanel{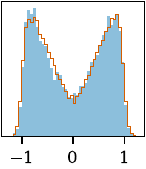} &
    \rbpanel{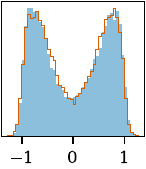} \\
    \rblabel{X_8(\tau)} &
    \rbpanel{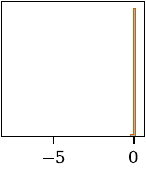} &
    \rbpanel{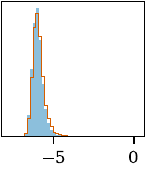} &
    \rbpanel{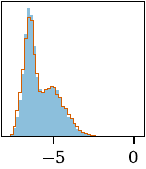} &
    \rbpanel{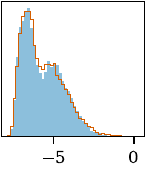} &
    \rbpanel{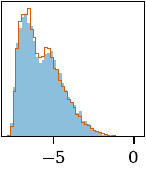} \\
    \rblabel{X_9(\tau)} &
    \rbpanel{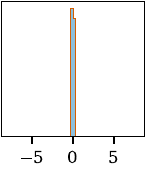} &
    \rbpanel{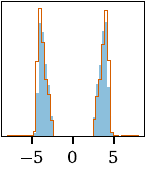} &
    \rbpanel{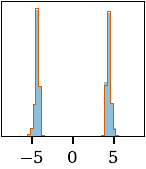} &
    \rbpanel{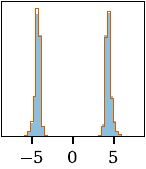} &
    \rbpanel{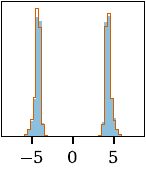}
  \end{tabular}
  \caption{Rayleigh--B\'{e}nard convection: marginal histograms of selected
  state-vector components $X_d(\tau)$ at five rollout times. BTJD (ours,
  \textcolor[HTML]{0072B2}{blue fill}) closely tracks the ground-truth
  distribution (\textcolor[HTML]{D55E00}{orange outline}).}
  \label{fig:RB:Marginal}
\end{figure}

\subsection{Stochastically forced Burgers equation} The preceding experiments establish transition-law and trajectory-level fidelity in systems where these quantities can be diagnosed directly. We now ask whether BTJD retains its advantage for high-dimensional stochastic dynamics stemming from stochastic PDEs, where inference costs become important. We first consider the Burgers equation with stochastic forcing.  

\subsubsection{Burgers: Setup and training data}\label{subsec:BurgSetup}
We first consider the one-dimensional Burgers equation with periodic boundary conditions on the spatial domain $[0,1)$,
    \begin{equation}\label{eq:BurgersDet}
        \partial _\tau u(\tau,x) = \nu \partial_{xx}u(\tau,x)-u(\tau,x)\partial_xu(\tau,x)
    \end{equation}
    with viscosity parameter $\nu =0.007$ controlling the width of shocks. We then discretize \eqref{eq:BurgersDet} in space on a uniformly spaced 64-point grid. This yields an ODE in $\R^{64}$, to which we add stochastic forcing. The forcing considered is discrete-in-space and not a standard Brownian motion, but colored noise. Specifically, for grid location $x_j$,  \begin{equation}\label{eq:ColorNoise}
    dW_j(\tau) = \sigma\sum_{\iota=1}^{10} \frac{1}{\iota}(a_{\iota}(x_j)dB_{\iota}^{(1)}(\tau)+b_{\iota}(x_j)dB_{\iota}^{(2)}(\tau))
    \end{equation}
    where each $B_{\iota}^{(1)}(\tau),B_{\iota}^{(2)}(\tau)$ represent independent Brownian motions and $a_{\iota},b_{\iota}$ represent sine and cosine modes. The noise amplitude is set to $\sigma = 0.04$. 
    
    Training initial conditions are drawn from Gaussian bumps with noise of the same structure as \eqref{eq:ColorNoise},
    \begin{equation}\label{eq:BurgersInit}
    u(0,x) = \exp(-20(x-0.5)^2)+0.015W(0,x)\,.
    \end{equation}
    We solve the forced \eqref{eq:BurgersDet} by evolving from initial conditions \eqref{eq:BurgersInit} using a method-of-lines discretization.  The solutions are computed on a uniformly spaced 64-point grid. The spatial derivatives are approximated by second-order centered finite differences.
The solutions are then evolved in time via Euler-Maruyama with $\Delta \tau =5 \times 10^{-4}$, on the interval $[0,4]$ yielding $8000$ step trajectories. The solution is then downsampled to $800$ steps. We generate $N = 4096$ training trajectories of this form.

We train a generator operating in a latent space. As such, a convolutional autoencoder is trained compressing the $64$-dimensional field into a $16$-dimensional state vector. The generator is then a convolutional neural network (CNN) of 3 FiLM-conditioned residual blocks \cite{perez2018film} of channel width $32$ with circular padding and SiLU activations. Time is injected through a learned embedding, and noise is concatenated channel-wise with the latent state. Additionally, we inject noise into the conditioning with a magnitude of $1\%$ of the standard deviation of the latent space feature magnitude. For this experiment, we train with $\varepsilon=0.1,h=0.1$, and 20 Sinkhorn iterations per step. Each gradient step $1024$ particles per timestep are drawn for four randomly chosen timesteps. We train for $100$k gradient steps using AdamW with learning rate $1\mathrm{e}{-3}$ and cosine schedule.

\subsubsection{Burgers: Test initial conditions and evaluation metrics}\label{sec:BurgersMetrics}
For test initial conditions, we draw an additional $4096$ test trajectories of the form \eqref{eq:BurgersInit}. We evaluate two quantities of interest on the stochastic Burgers trajectories, the energy $E(\tau)$ and enstrophy $Z(\tau)$ at each time point,
\begin{align}\label{eq:BurgEnergyEnstrophy}
        &E(\tau) = \frac{1}{2} \int_0^1 |u(\tau,x)|^2 dx\,,\qquad Z(\tau) = \frac{1}{2} \int_0^1 |\partial_xu(\tau,x)|^2 dx\,.
    \end{align}
    Both integrals are approximated using the trapezoidal rule, with $\partial_xu(\tau,x)$ approximated using centered second-order finite differences on the 64-point grid.
    We assess generated samples via the relative absolute errors of the energy and enstrophy averaged over all time steps, and also report the standard error of the mean relative error across time steps.

\subsubsection{Burgers: Results}
\Cref{fig:stochBurgersSnapshots} compares BTJD samples with ground-truth solutions of the stochastic Burgers equation at several time steps. Despite using a single model evaluation for each physical time step, BTJD accurately predicts the stochastic variability of the solution ensemble together with the sharp spatial structures generated by the nonlinear dynamics.
The quantitative results in \Cref{tab:burgers} further support that BTJD generates accurate trajectories. BTJD achieves the lowest reported mean errors in both energy and enstrophy while requiring only one neural-network function evaluation (NFE) per physical time step, compared with $20$ evaluations for conditional flow matching (CFM) and $50$--$100$ denoising steps for the autoregressive diffusion models (ARDM). Its enstrophy error is less than half that of the next-best tested method.

\begin{figure}[tb]
  \centering \setlength{\tabcolsep}{2pt}
  \begin{tabular}{cccc}
    \hspace*{-0.5cm}\includegraphics[width=0.25\textwidth]{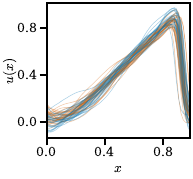} &\hspace*{-0.14cm}
    \includegraphics[width=0.25\textwidth]{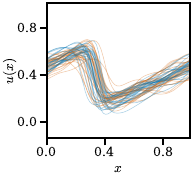} &\hspace*{-0.14cm}
    \includegraphics[width=0.25\textwidth]{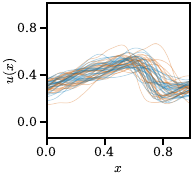} &\hspace*{-0.14cm}
    \includegraphics[width=0.25\textwidth]{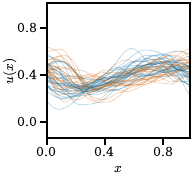} \\
    (a) $\tau = 0.5$ & \hspace*{-0.1cm}(b) $\tau = 1.5$ & \hspace*{-0.1cm}(c) $\tau = 2.5$ & \hspace*{-0.1cm}(d) $\tau = 3.5$
  \end{tabular}
  \caption{Stochastic Burgers: BTJD generates trajectories that capture both stochastic variability and sharp spatial structures over time; ground-truth samples (\textcolor[HTML]{D55E00}{vermillion}) and BTJD (ours) samples (\textcolor[HTML]{0072B2}{blue}).}
  \label{fig:stochBurgersSnapshots}
\end{figure}

\begin{table}
\centering\small
\caption{Stochastic Burgers: With one model evaluation per time step, BTJD achieves the lowest energy and enstrophy errors, with more than a twofold reduction in enstrophy error over the next-best method.}
\label{tab:burgers}
\resizebox{\textwidth}{!}{
\begin{tabular}{lccc}
\toprule
\textbf{Method} & error energy & error enstrophy & NFEs/time step \\
\midrule
Operator learning \cite{stachenfeld2021learned}
& $2.21\mathrm{e}{-2}\;(\pm\,1.54\mathrm{e}{-2})$
& $2.55\mathrm{e}{-1}\;(\pm\,2.34\mathrm{e}{-1})$
& 1 \\

ARDM 50 steps\cite{KOHL2026108641}
& $1.46\mathrm{e}{-2}\;(\pm\,3.39\mathrm{e}{-3})$
& $2.49\mathrm{e}{-1}\;(\pm\,1.43\mathrm{e}{-1})$
& 50 \\

ARDM 75 steps \cite{KOHL2026108641}
& $1.36\mathrm{e}{-2}\;(\pm\,3.49\mathrm{e}{-3})$
& $2.30\mathrm{e}{-1}\;(\pm\,1.26\mathrm{e}{-1})$
& 75 \\

ARDM 100 steps \cite{KOHL2026108641}
& $1.24\mathrm{e}{-2}\;(\pm\,3.12\mathrm{e}{-3})$
& $2.11\mathrm{e}{-1}\;(\pm\,1.14\mathrm{e}{-1})$
& 100 \\
CFM 20 steps \cite{albergo-vanden-eijnden:2022,lipman2023flow} & $2.71\mathrm{e}{-3}\;(\pm\,1.96\mathrm{e}{-3})$
& $1.53\mathrm{e}{-1}\;(\pm\,1.28\mathrm{e}{-1})$
& 20 \\
MeanFlow 1 step\cite{geng2025mean} & $8.22\mathrm{e}{-1}\;(\pm\,2.88\mathrm{e}{-1})$ 
 & $3.59\mathrm{e}{+2}\;(\pm\,2.78\mathrm{e}{+2})$ 
 & 1 \\
MeanFlow 2 steps \cite{geng2025mean} & $1.69\mathrm{e}{-1}\;(\pm\,5.77\mathrm{e}{-2})$ 
& $9.81\mathrm{e}{+1}\;(\pm\,8.05\mathrm{e}{+1})$ 
& 2 \\
ReFlow+Distill 1 step\cite{liu2023rectified} & $2.78\mathrm{e}{-3}\;(\pm\,2.10\mathrm{e}{-3})$ & $1.40\mathrm{e}{-1}\;(\pm\,1.31\mathrm{e}{-1})$ & 1 \\
\textbf{BTJD (ours)}
& $\mathbf{2.22\mathrm{e}{-3}}\;(\pm\,2.47\mathrm{e}{-3})$
& $\mathbf{5.87\mathrm{e}{-2}}\;(\pm\,1.01\mathrm{e}{-1})$
& 1 \\
\bottomrule
\end{tabular}
}
\end{table}

\subsection{Stochastically forced two-dimensional turbulence}
We now consider a stochastically forced two-dimensional incompressible flow, providing a high-dimensional problem with chaotic multiscale dynamics.

\begin{figure}
\centering
\includegraphics[width=1.0\textwidth]{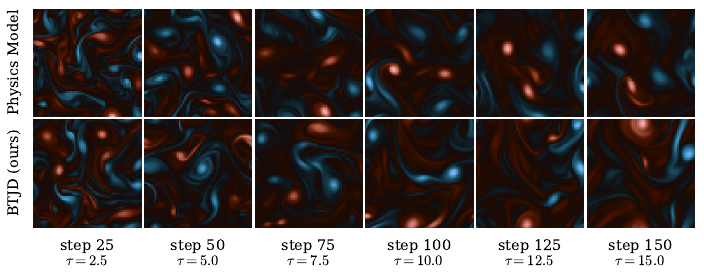}
\caption{Turbulence: Our BTJD generates statistically representative turbulent trajectories with only one neural-network function evaluation per time step.}
\label{fig:NSIndividual}
\end{figure}

\subsubsection{Turbulence: Setup and training data}
We consider the two-dimensional incompressible Navier-Stokes equations on the periodic domain $\Omega=[0,2\pi)^2$, adapting the setup of \cite{Kochkov2021-ML-CFD}. In vorticity form, we have the equation, 
    \begin{equation}\label{eq:NSDef}
        \partial_\tau \omega
        +
        (u\cdot\nabla\omega)
        =
        \left(\nu\Delta\omega-\alpha\omega\right),
        \qquad
        u=\nabla^\perp\Delta^{-1}\omega,
    \end{equation}
    where $\nu=10^{-3}$ is the viscosity and $\alpha=0.1$ is a linear drag coefficient. We discretize on a uniform grid, yielding an ODE, to which we introduce stochastic forcing through low-frequency Fourier modes, at location $x_j$, for $\mathcal{K}_{\textrm{sto}} = \{k \in \mathbb{Z}^2: 0 <|k| \leq 4 \}$:
    \begin{equation}\label{eq:NSForcing}
       dW_j(\tau)
        =
        \sigma
        \sum_{\kappa\in \mathcal{K}_{\textrm{sto}}}
        w_\kappa
        \left(
            a_\kappa(x_j)\,dB_\kappa^{(1)}(\tau)
            +
            b_\kappa(x_j)\,dB_\kappa^{(2)}(\tau)
        \right),
    \end{equation}
    where $a_\kappa= \cos(\kappa \cdot x)$ and $b_\kappa = \sin (\kappa \cdot x)$ are sine and cosine modes, $w_\kappa \propto|\kappa|^{-\frac{1}{2}}$ are the spectral weights, and the Brownian motions are mutually independent. We set $\sigma=0.3$ and add this forcing to each grid point at each step of the integration.

 Initial conditions are sampled from the ensemble used in \cite{Kochkov2021-ML-CFD} and rescaled so that the maximum vorticity is approximately seven as in \cite{Kochkov2021-ML-CFD}.
We solve \eqref{eq:NSDef} on a $256\times256$ grid using a pseudo-spectral method with $2/3$-rule dealiasing. The linear terms are integrated with Crank-Nicolson and the nonlinear term with a fourth-order Runge-Kutta scheme.
We generate $N=2048$ training trajectories using $\Delta \tau=0.001$ and $25{,}000$ fine time steps. Each trajectory is temporally subsampled to $250$ states and spectrally subsampled to a $64\times64$ grid for training.

For generation on high dimensional data, we follow the latent space embedding and Masked Autoencoder (MAE) feature extraction of \cite{deng2026drifting,han2026wflow}. Both the MAE and latent space feature extractor are autoencoders, and the generator is a 2D U-Net with channel widths of $256$ and $512$. Time is  injected via a learned embedding and noise is concatenated channel-wise. Additionally, the conditioning input is perturbed by Gaussian noise of $2\%$ magnitude of the standard deviation of the latent space  magnitude as in \ref{subsec:BurgSetup}. We train with $\varepsilon = 0.07$, which is then scaled by the dimension of the feature as well as the mean inter-particle distance as in \cite{deng2026drifting}. We set $h=1$ and use $20$ Sinkhorn iterations per-step. Each gradient step samples $96$ particles per timestep for $64$ sampled timesteps. We train for $65,000$ gradient steps of AdamW with learning rate $1\mathrm{e}{-4}$ and $5000$-step warmup.

\subsubsection{Turbulence: Test initial conditions and evaluation metrics}
We generate $1024$ additional test trajectories from independently sampled initial conditions following the same construction as the training data.
We assess the predicted dynamics using the kinetic energy and enstrophy,
    \begin{equation}\label{eq:NSEnergyEnstrophy}
        E(\tau)
        =
        \frac{1}{2}\int_\Omega |u(\tau,x)|^2\,dx,
        \qquad
        Z(\tau)
        =
        \frac{1}{2}\int_\Omega |\omega(\tau,x)|^2\,dx.
    \end{equation}
    These quantities characterize the evolution of the kinetic energy and the strength of the vortical structures, respectively. We report their mean relative errors over all timesteps along with the standard error of this quantity over timesteps as in \Cref{sec:BurgersMetrics}.

\subsubsection{Turbulence: Results}
\Cref{fig:NSIndividual,fig:NSEnsemble} show that BTJD generates accurate individual turbulent trajectories as well as diverse ensembles of stochastic  realizations. Thus, the direct one-step sampler remains expressive enough to represent the variability and multiscale structures of the stochastically forced flow dynamics.
The quantitative comparison in \Cref{tab:jaxcfd_results} highlights that BTJD achieves the lowest reported mean errors in both kinetic energy and enstrophy, while requiring only a single neural-network function evaluation (NFE) per time step. Relative to the next-best reported mean errors, BTJD reduces the energy error by approximately a factor of four and the enstrophy error by approximately a factor of three. It therefore outperforms both multi-step diffusion and flow baselines and one-step distilled models.

\begin{figure}
\centering
\includegraphics[width=1.0\textwidth]{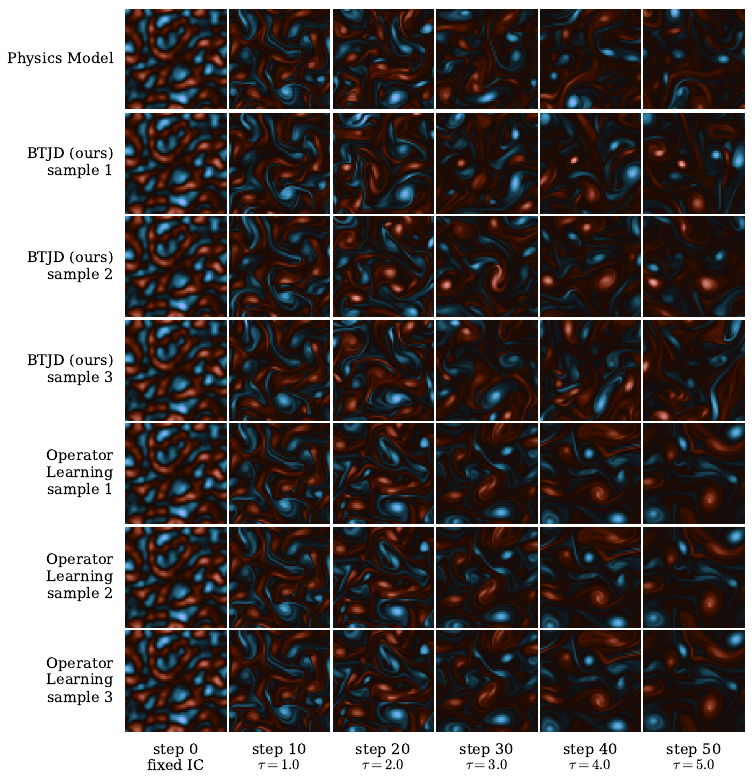}
\caption{Turbulence: BTJD captures the intrinsic stochasticity of the dynamics. Starting from the same initial condition, repeated BTJD rollouts separate and produce distinct turbulent realizations, as seen in the three generated trajectories. In contrast, a deterministic surrogate such as given by operator learning returns the same trajectory from the same initial condition and therefore cannot represent this stochastic variability.}
\label{fig:NSEnsemble}
\end{figure}

\section{Conclusions}
We developed BTJD for learning one-step generative surrogate models of stochastic dynamics from trajectory data. By drifting the joint law of consecutive states while preserving the current-state marginal, the method turns available transition pairs into a direct conditional sampler and enables stochastic rollouts with one model evaluation per time step. The underlying idea extends beyond the particular drifting scheme considered here. More generally, for generative procedures that require the target distribution to enter the learning update through an empirical distribution given by samples, lifting the learning problem to the joint space can make the target distribution accessible from trajectory data, provided that the model is endowed with sufficient structure to expose the desired conditional after training. The role of the block-triangular parametrization in BTJD is precisely to provide this structure.

\begin{table}
\centering\small
\caption{Turbulence: With one neural-network function evaluation (NFE) per time step, BTJD reduces energy error by about a factor four and enstrophy error by almost a factor of three over the next-best baseline.}
\label{tab:jaxcfd_results}
\centering
\resizebox{\textwidth}{!}{
\begin{tabular}{lccc}
\toprule
\textbf{Method} & error energy & error enstrophy & NFEs/time step \\
\midrule
Operator learning \cite{stachenfeld2021learned}
& $\pmc{2.13\mathrm{e}{-1}}{1.58\mathrm{e}{-1}}$
& $\pmc{1.92\mathrm{e}{-1}}{1.48\mathrm{e}{-1}}$
& 1 \\

ARDM 50 steps \cite{KOHL2026108641}
& $\pmc{3.14\mathrm{e}{-1}}{1.48\mathrm{e}{-1}}$
& $\pmc{2.82\mathrm{e}{-1}}{9.60\mathrm{e}{-2}}$
& $50$ \\

ARDM 75 steps \cite{KOHL2026108641}
& $\pmc{2.75\mathrm{e}{-1}}{1.40\mathrm{e}{-1}}$
& $\pmc{1.71\mathrm{e}{-1}}{9.40\mathrm{e}{-2}}$
& $75$ \\
ARDM 100 steps \cite{KOHL2026108641}
& $\pmc{1.16\mathrm{e}{-1}}{1.09\mathrm{e}{-1}}$
& $\pmc{1.34\mathrm{e}{-1}}{8.46\mathrm{e}{-2}}$
& $100$ \\
CFM 20 steps \cite{albergo-vanden-eijnden:2022,lipman2023flow}
& $\pmc{1.39\mathrm{e}{-1}}{6.30\mathrm{e}{-2}}$
& $\pmc{1.04\mathrm{e}{-1}}{7.70\mathrm{e}{-2}}$
& 20 \\
MeanFlow 1 step \cite{geng2025mean}
& $\pmc{2.46\mathrm{e}{-1}}{1.25\mathrm{e}{-1}}$
& $\pmc{4.37\mathrm{e}{-1}}{2.78\mathrm{e}{-1}}$
& $1$ \\
MeanFlow 2 steps \cite{geng2025mean}
& $\pmc{1.40\mathrm{e}{-1}}{6.61\mathrm{e}{-2}}$
& $\pmc{9.81\mathrm{e}{-2}}{2.97\mathrm{e}{-2}}$
& $2$ \\
MeanFlow 4 steps \cite{geng2025mean}
& $\pmc{1.22\mathrm{e}{-1}}{8.45\mathrm{e}{-2}}$
& $\pmc{6.74\mathrm{e}{-2}}{3.66\mathrm{e}{-2}}$
& $4$ \\
ReFlow+Distill 1 step \cite{liu2023rectified}
& $\pmc{8.60\mathrm{e}{-2}}{4.41\mathrm{e}{-2}}$
& $\pmc{7.16\mathrm{e}{-2}}{6.36\mathrm{e}{-2}}$
& $1$ \\
\textbf{BTJD (ours)}
& $\pmc{\mathbf{2.14\mathrm{e}{-2}}}{1.93\mathrm{e}{-2}}$
& $\pmc{\mathbf{2.52\mathrm{e}{-2}}}{1.91\mathrm{e}{-2}}$
& 1 \\
\bottomrule
\end{tabular}
}
\end{table}

\bibliography{main} 

\begin{thebibliography}{10}

\bibitem{albergo2025stochastic}
M.~S. Albergo, N.~M. Boffi, and E.~Vanden-Eijnden.
\newblock Stochastic interpolants: A unifying framework for flows and diffusions.
\newblock {\em Journal of Machine Learning Research}, 26(209):1--80, 2025.

\bibitem{albergo-vanden-eijnden:2022}
M.~S. Albergo and E.~Vanden-Eijnden.
\newblock Building normalizing flows with stochastic interpolants.
\newblock In {\em The Eleventh International Conference on Learning Representations}, 2023.

\bibitem{Ambrosio2013}
L.~Ambrosio and N.~Gigli.
\newblock A user's guide to optimal transport.
\newblock In {\em Modelling and Optimisation of Flows on Networks: Cetraro, Italy 2009}, volume 2062 of {\em Lecture Notes in Mathematics}, pages 1--155. Springer, 2013.

\bibitem{baptista2024conditional}
R.~Baptista, B.~Hosseini, N.~B. Kovachki, and Y.~M. Marzouk.
\newblock Conditional sampling with monotone {GAN}s: From generative models to likelihood-free inference.
\newblock {\em SIAM/ASA Journal on Uncertainty Quantification}, 12(3):868--900, 2024.

\bibitem{Baptista2024}
R.~Baptista, Y.~M. Marzouk, and O.~Zahm.
\newblock On the representation and learning of monotone triangular transport maps.
\newblock {\em Foundations of Computational Mathematics}, 24(6):2063--2108, 2024.

\bibitem{bartosh2025sdematching}
G.~Bartosh, D.~Vetrov, and C.~A. Naesseth.
\newblock {SDE} matching: Scalable and simulation-free training of latent stochastic differential equations.
\newblock In {\em Proceedings of the 42nd International Conference on Machine Learning}, volume 267 of {\em Proceedings of Machine Learning Research}, pages 3054--3070. PMLR, 2025.

\bibitem{BATLLE2024112549}
P.~Batlle, M.~Darcy, B.~Hosseini, and H.~Owhadi.
\newblock Kernel methods are competitive for operator learning.
\newblock {\em Journal of Computational Physics}, 496:112549, 2024.

\bibitem{berman2024parametric}
J.~Berman, T.~Blickhan, and B.~Peherstorfer.
\newblock Parametric model reduction of mean-field and stochastic systems via higher-order action matching.
\newblock In {\em Advances in Neural Information Processing Systems}, volume~37, pages 56588--56618, 2024.

\bibitem{berman2026ngif}
J.~Berman, T.~Blickhan, and B.~Peherstorfer.
\newblock Leveraging gauge freedom for learning non-gradient population dynamics of stochastic systems.
\newblock In {\em Proceedings of the 43rd International Conference on Machine Learning}, volume 306 of {\em Proceedings of Machine Learning Research}. PMLR, 2026.

\bibitem{berman2026stochastic}
J.~Berman, T.~Blickhan, and B.~Peherstorfer.
\newblock Stochastic lifting for generating trajectories of stochastic physical systems.
\newblock In {\em Proceedings of the 43rd International Conference on Machine Learning}, volume 306 of {\em Proceedings of Machine Learning Research}. PMLR, 2026.

\bibitem{blickhan-berman-stuart-etal:2025}
T.~Blickhan, J.~Berman, A.~M. Stuart, and B.~Peherstorfer.
\newblock {DICE}: Discrete inverse continuity equation for learning population dynamics, 2025.

\bibitem{boffi2025how}
N.~M. Boffi, M.~S. Albergo, and E.~Vanden-Eijnden.
\newblock How to build a consistency model: Learning flow maps via self-distillation.
\newblock In {\em Advances in Neural Information Processing Systems}, volume~38, 2025.

\bibitem{cachay2023dyffusion}
S.~R. Cachay, B.~Zhao, H.~Joren, and R.~Yu.
\newblock {DYffusion}: A dynamics-informed diffusion model for spatiotemporal forecasting.
\newblock In {\em Advances in Neural Information Processing Systems}, volume~36, 2023.

\bibitem{chen2024probforecast}
Y.~Chen, M.~Goldstein, M.~Hua, M.~S. Albergo, N.~M. Boffi, and E.~Vanden-Eijnden.
\newblock Probabilistic forecasting with stochastic interpolants and {F{\"o}llmer} processes.
\newblock In {\em Proceedings of the 41st International Conference on Machine Learning}, volume 235 of {\em Proceedings of Machine Learning Research}, pages 6728--6756. PMLR, 2024.

\bibitem{conti2026veni}
P.~Conti, J.~Kneifl, A.~Manzoni, A.~Frangi, J.~Fehr, S.~L. Brunton, and J.~N. Kutz.
\newblock {VENI, VINDy, VICI}: A generative reduced-order modeling framework with uncertainty quantification.
\newblock {\em Neural Networks}, 198:108543, 2026.

\bibitem{deng2026drifting}
M.~Deng, H.~Li, T.~Li, Y.~Du, and K.~He.
\newblock Generative modeling via drifting, 2026.

\bibitem{dridi2021learning}
N.~Dridi, L.~Drumetz, and R.~Fablet.
\newblock Learning stochastic dynamical systems with neural networks mimicking the {Euler--Maruyama} scheme.
\newblock In {\em 2021 29th European Signal Processing Conference (EUSIPCO)}, pages 1990--1994. IEEE, 2021.

\bibitem{pmlr-v89-feydy19a}
J.~Feydy, T.~S{\'e}journ{\'e}, F.-X. Vialard, S.-i. Amari, A.~Trouv{\'e}, and G.~Peyr{\'e}.
\newblock Interpolating between optimal transport and {MMD} using sinkhorn divergences.
\newblock In {\em Proceedings of the Twenty-Second International Conference on Artificial Intelligence and Statistics}, volume~89 of {\em Proceedings of Machine Learning Research}, pages 2681--2690. PMLR, 2019.

\bibitem{freitag2025learning}
M.~A. Freitag, J.~M. Nicolaus, and M.~Redmann.
\newblock Learning stochastic reduced models from data: A nonintrusive approach.
\newblock {\em SIAM Journal on Scientific Computing}, 47(5):A2851--A2880, 2025.

\bibitem{Fresca2021}
S.~Fresca, L.~Dede', and A.~Manzoni.
\newblock A comprehensive deep learning-based approach to reduced order modeling of nonlinear time-dependent parametrized pdes.
\newblock {\em Journal of Scientific Computing}, 87(2):61, Apr 2021.

\bibitem{geng2025mean}
Z.~Geng, M.~Deng, X.~Bai, J.~Z. Kolter, and K.~He.
\newblock Mean flows for one-step generative modeling.
\newblock In {\em Advances in Neural Information Processing Systems}, volume~38, 2025.

\bibitem{gloeckler2024allinone}
M.~Gloeckler, M.~Deistler, C.~D. Weilbach, F.~Wood, and J.~H. Macke.
\newblock All-in-one simulation-based inference.
\newblock In {\em Proceedings of the 41st International Conference on Machine Learning}, volume 235 of {\em Proceedings of Machine Learning Research}, pages 15735--15766. PMLR, 2024.

\bibitem{han2026wflow}
J.~Han, P.~Li, Q.~Guo, R.~Xu, S.~Ermon, and E.~J. Cand{\`e}s.
\newblock One-step generative modeling via wasserstein gradient flows, 2026.

\bibitem{ho2020denoising}
J.~Ho, A.~N. Jain, and P.~Abbeel.
\newblock Denoising diffusion probabilistic models.
\newblock In {\em Advances in Neural Information Processing Systems}, volume~33, 2020.

\bibitem{jha2026firstordertrajectorymatchingfast}
S.~Jha, T.~Schorlepp, N.~Geissler, J.~Berman, and B.~Peherstorfer.
\newblock First-order trajectory matching: Fast ensemble predictions of chaotic, turbulent, stochastic systems.
\newblock {\em arXiv}, 2606.11138, 2026.

\bibitem{kerrigan2024dynamic}
G.~Kerrigan, G.~Migliorini, and P.~Smyth.
\newblock Dynamic conditional optimal transport through simulation-free flows.
\newblock In {\em Advances in Neural Information Processing Systems}, volume~37, pages 93602--93642, 2024.

\bibitem{Kochkov2021-ML-CFD}
D.~Kochkov, J.~A. Smith, A.~Alieva, Q.~Wang, M.~P. Brenner, and S.~Hoyer.
\newblock Machine learning--accelerated computational fluid dynamics.
\newblock {\em Proceedings of the National Academy of Sciences}, 118(21):e2101784118, 2021.

\bibitem{KOHL2026108641}
G.~Kohl, L.-W. Chen, and N.~Thuerey.
\newblock Benchmarking autoregressive conditional diffusion models for turbulent flow simulation.
\newblock {\em Neural Networks}, 199:108641, 2026.

\bibitem{JMLR:v24:21-1524}
N.~Kovachki, Z.~Li, B.~Liu, K.~Azizzadenesheli, K.~Bhattacharya, A.~Stuart, and A.~Anandkumar.
\newblock Neural operator: Learning maps between function spaces with applications to pdes.
\newblock {\em Journal of Machine Learning Research}, 24(89):1--97, 2023.

\bibitem{lai2026unifiedviewdriftingscorebased}
C.-H. Lai, B.~Nguyen, N.~Murata, Y.~Takida, T.~Uesaka, Y.~Mitsufuji, S.~Ermon, and M.~Tao.
\newblock A unified view of drifting and score-based models, 2026.

\bibitem{lipman2023flow}
Y.~Lipman, R.~T.~Q. Chen, H.~Ben-Hamu, M.~Nickel, and M.~Le.
\newblock Flow matching for generative modeling.
\newblock In {\em The Eleventh International Conference on Learning Representations}, 2023.

\bibitem{liu2023rectified}
X.~Liu, C.~Gong, and Q.~Liu.
\newblock Flow straight and fast: Learning to generate and transfer data with rectified flow.
\newblock In {\em The Eleventh International Conference on Learning Representations}, 2023.

\bibitem{Lu2021}
L.~Lu, P.~Jin, G.~Pang, Z.~Zhang, and G.~E. Karniadakis.
\newblock Learning nonlinear operators via deeponet based on the universal approximation theorem of operators.
\newblock {\em Nature Machine Intelligence}, 3(3):218--229, Mar 2021.

\bibitem{Marzouk2017}
Y.~Marzouk, T.~Moselhy, M.~Parno, and A.~Spantini.
\newblock Sampling via measure transport: An introduction.
\newblock In R.~Ghanem, D.~Higdon, and H.~Owhadi, editors, {\em Handbook of Uncertainty Quantification}, pages 785--825. Springer, 2017.

\bibitem{neklyudov2023action}
K.~Neklyudov, R.~Brekelmans, D.~Severo, and A.~Makhzani.
\newblock Action matching: Learning stochastic dynamics from samples.
\newblock In {\em Proceedings of the 40th International Conference on Machine Learning}, volume 202 of {\em Proceedings of Machine Learning Research}, pages 25858--25889. PMLR, 2023.

\bibitem{otness2021an}
K.~Otness, A.~Gjoka, J.~Bruna, D.~Panozzo, B.~Peherstorfer, T.~Schneider, and D.~Zorin.
\newblock An extensible benchmark suite for learning to simulate physical systems.
\newblock In {\em Advances in Neural Information Processing Systems, Datasets and Benchmarks Track}, 2021.

\bibitem{PEHERSTORFER2016196}
B.~Peherstorfer and K.~Willcox.
\newblock Data-driven operator inference for nonintrusive projection-based model reduction.
\newblock {\em Computer Methods in Applied Mechanics and Engineering}, 306:196--215, 2016.

\bibitem{perez2018film}
E.~Perez, F.~Strub, H.~de~Vries, V.~Dumoulin, and A.~Courville.
\newblock {FiLM}: Visual reasoning with a general conditioning layer.
\newblock In {\em Proceedings of the Thirty-Second AAAI Conference on Artificial Intelligence (AAAI-18)}, volume~32, pages 3942--3951, 2018.

\bibitem{Regazzoni2024}
F.~Regazzoni, S.~Pagani, M.~Salvador, L.~Dede', and A.~Quarteroni.
\newblock Learning the intrinsic dynamics of spatio-temporal processes through latent dynamics networks.
\newblock {\em Nature Communications}, 15(1):1834, Feb 2024.

\bibitem{reiterer-lainscek-schuerrer-etal:1998}
P.~Reiterer, C.~Lainscsek, F.~Sch{\"u}rrer, C.~Letellier, and J.~Maquet.
\newblock A nine-dimensional lorenz system to study high-dimensional chaos.
\newblock {\em Journal of Physics A: Mathematical and General}, 31(34):7121--7139, 1998.

\bibitem{salimans2022progressive}
T.~Salimans and J.~Ho.
\newblock Progressive distillation for fast sampling of diffusion models.
\newblock In {\em The Tenth International Conference on Learning Representations}, 2022.

\bibitem{santambrogio2015optimal}
F.~Santambrogio.
\newblock {\em Optimal Transport for Applied Mathematicians: Calculus of Variations, {PDE}s, and Modeling}, volume~87 of {\em Progress in Nonlinear Differential Equations and Their Applications}.
\newblock Birkh{\"a}user, 2015.

\bibitem{schwerdtner2026twoparameter}
P.~Schwerdtner, T.~Blickhan, and B.~Peherstorfer.
\newblock Two-parameter flows for learning population dynamics of physical systems.
\newblock In {\em Proceedings of the 43rd International Conference on Machine Learning}, volume 306 of {\em Proceedings of Machine Learning Research}. PMLR, 2026.

\bibitem{pmlr-v180-shi22a}
Y.~Shi, V.~De~Bortoli, G.~Deligiannidis, and A.~Doucet.
\newblock Conditional simulation using diffusion {Schr{\"o}dinger} bridges.
\newblock In {\em Proceedings of the Thirty-Eighth Conference on Uncertainty in Artificial Intelligence}, volume 180 of {\em Proceedings of Machine Learning Research}, pages 1792--1802. PMLR, 2022.

\bibitem{shysheya2024conditional}
A.~Shysheya, C.~Diaconu, F.~Bergamin, P.~Perdikaris, J.~M. Hern{\'a}ndez-Lobato, R.~E. Turner, and E.~Mathieu.
\newblock On conditional diffusion models for {PDE} simulations.
\newblock In {\em Advances in Neural Information Processing Systems}, volume~37, pages 23246--23300, 2024.

\bibitem{sohldickstein2015deep}
J.~Sohl-Dickstein, E.~Weiss, N.~Maheswaranathan, and S.~Ganguli.
\newblock Deep unsupervised learning using nonequilibrium thermodynamics.
\newblock In {\em Proceedings of the 32nd International Conference on Machine Learning}, volume~37 of {\em Proceedings of Machine Learning Research}, pages 2256--2265. PMLR, 2015.

\bibitem{song2023consistency}
Y.~Song, P.~Dhariwal, M.~Chen, and I.~Sutskever.
\newblock Consistency models.
\newblock In {\em Proceedings of the 40th International Conference on Machine Learning}, volume 202 of {\em Proceedings of Machine Learning Research}, pages 32211--32252. PMLR, 2023.

\bibitem{song2021score}
Y.~Song, J.~Sohl-Dickstein, D.~P. Kingma, A.~Kumar, S.~Ermon, and B.~Poole.
\newblock Score-based generative modeling through stochastic differential equations.
\newblock In {\em The Ninth International Conference on Learning Representations}, 2021.

\bibitem{doi:10.1137/20M1312204}
A.~Spantini, R.~Baptista, and Y.~M. Marzouk.
\newblock Coupling techniques for nonlinear ensemble filtering.
\newblock {\em SIAM Review}, 64(4):921--953, 2022.

\bibitem{JMLR:v19:17-747}
A.~Spantini, D.~Bigoni, and Y.~M. Marzouk.
\newblock Inference via low-dimensional couplings.
\newblock {\em Journal of Machine Learning Research}, 19(66):1--71, 2018.

\bibitem{stachenfeld2021learned}
K.~Stachenfeld, D.~B. Fielding, D.~Kochkov, M.~Cranmer, T.~Pfaff, J.~Godwin, C.~Cui, S.~Ho, P.~Battaglia, and {\'A}.~S{\'a}nchez-Gonz{\'a}lez.
\newblock Learned coarse models for efficient turbulence simulation.
\newblock In {\em The Tenth International Conference on Learning Representations}, 2022.

\bibitem{tong2024conditional}
A.~Tong, K.~Fatras, N.~Malkin, G.~Huguet, Y.~Zhang, J.~Rector-Brooks, G.~Wolf, and Y.~Bengio.
\newblock Improving and generalizing flow-based generative models with minibatch optimal transport.
\newblock {\em Transactions on Machine Learning Research}, pages 1--34, 2024.

\bibitem{turan2026generativedriftingsecretlyscore}
E.~Turan and M.~Ovsjanikov.
\newblock Generative drifting is secretly score matching: A spectral and variational perspective, 2026.

\bibitem{wyrod2026generative}
P.~Wyrod, A.~Chattopadhyay, and D.~Venturi.
\newblock Generative forecasting with joint probability models.
\newblock {\em Journal of Computational Physics}, 563:115109, 2026.

\bibitem{zhai2026conditional}
P.~S. Zhai, S.~Jeong, and V.~Ro{\v{c}}kov{\'a}.
\newblock Conditional flow matching for bayesian posterior inference.
\newblock In {\em Proceedings of the 29th International Conference on Artificial Intelligence and Statistics}, volume 300 of {\em Proceedings of Machine Learning Research}, pages 2044--2052. PMLR, 2026.

\bibitem{zhou2025imm}
L.~Zhou, S.~Ermon, and J.~Song.
\newblock Inductive moment matching.
\newblock In {\em Proceedings of the 42nd International Conference on Machine Learning}, volume 267 of {\em Proceedings of Machine Learning Research}, pages 78651--78686. PMLR, 2025.

\end{thebibliography}

\end{document}